\documentclass[final]{l4dc2026}

\usepackage{gen_settings}

\title[Scalable Verification of Neural Control Barrier Functions Using LBP]{Scalable Verification of Neural Control Barrier Functions\\[1ex] Using Linear Bound Propagation}
\usepackage{times}

\author{%
 \Name{Nikolaus Vertovec}${}^{1,2}$ \Email{nikolaus.vertovec@st-hughs.ox.ac.uk}
 \AND
 \Name{Frederik Baymler Mathiesen}${}^{1,3}$ \Email{frederik@baymler.com}
 \AND
 \Name{Thom Badings}${}^{1,4}$ \Email{thom.badings@cs.rwth-aachen.de}
 \AND
 \Name{Luca Laurenti}${}^{3,5}$ \Email{l.laurenti@tudelft.nl}
 \AND
 \Name{Alessandro Abate}${}^1$ \Email{alessandro.abate@cs.ox.ac.uk}\,\,%
 \\
 \addr ${}^1$Department of Computer Science, University of Oxford, Oxford, United Kingdom%
 \\
 \addr ${}^2$St Hugh's College, University of Oxford, Oxford, United Kingdom%
 \\
 \addr ${}^3$Delft Center for Systems and Control, TU Delft, Delft, The Netherlands%
 \\
 \addr ${}^4$Dep. of Computer Science \& Inst. for Data Science in Mech. Eng., RWTH Aachen University, Germany%
 \\
 \addr ${}^5$AI4I, Italy%
}

\usepackage{enumitem}   
\setlist[enumerate]{%
        itemsep=0em, 
        topsep=0mm plus 2mm}
\setlist[itemize]{%
        itemsep=0em, 
        topsep=0mm plus 2mm}

\usepackage{todonotes}

\makeatletter
\newcommand{\listtodos}[1][\@todonotes@todolistname]{\section*{#1}\@starttoc{tdo}}

\begin{document}

    \maketitle

\begin{abstract}%
Control barrier functions (CBFs) are a popular tool for safety certification of nonlinear dynamical control systems. Recently, CBFs represented as neural networks have shown great promise due to their expressiveness and applicability to a broad class of dynamics and safety constraints. However, verifying that a trained neural network is indeed a valid CBF is a computational bottleneck that limits the size of the networks that can be used. To overcome this limitation, we present a novel framework for verifying neural CBFs based on piecewise linear upper and lower bounds on the conditions required for a neural network to be a CBF. Our approach is rooted in linear bound propagation (LBP) for neural networks, which we extend to compute bounds on the gradients of the network. Combined with McCormick relaxation, we derive linear upper and lower bounds on the CBF conditions, thereby eliminating the need for computationally expensive verification procedures. Our approach applies to arbitrary control-affine systems and a broad range of nonlinear activation functions. To reduce conservatism, we develop a parallelizable refinement strategy that adaptively refines the regions over which these bounds are computed. Our approach scales to larger neural~networks than state-of-the-art verification procedures for CBFs, as demonstrated by our numerical experiments. 
\end{abstract}


\section{Introduction}\label{sec:intro}
Safety verification of autonomous control systems---such as unmanned aerial and ground vehicles or robotic manipulators---is critical for their deployment in real-world environments~\citep{DBLP:journals/arcras/HsuHF24}. As a result, certifying the safety of these systems, commonly modelled as nonlinear dynamical control systems, has become increasingly important. To certify safety, \emph{control barrier functions}~(CBFs) have emerged as an effective tool for enforcing \emph{forward control invariance}. CBFs have been successfully applied to robotic manipulators~\citep{DBLP:journals/tcst/Shaw-CortezOMC21}, vehicle cruise control~\citep{DBLP:conf/cdc/AmesGT14}, bipedal robots~\citep{DBLP:conf/rss/AgrawalS17}, and satellite navigation~\citep{DBLP:journals/automatica/BreedenP23}. 

Finding a CBF---particularly one that maximizes the \emph{control invariant set}---remains an ongoing challenge in the control literature~\citep{DBLP:journals/tac/AmesXGT17,DBLP:journals/tac/XiaoB22,DBLP:conf/cdc/Clark21}. A common approach relies on \emph{sum-of-squares} (SOS) programming~\citep{Han_SOS_2023}, which formulates safety and invariance conditions as semidefinite programs. However, SOS methods generally only apply when the dynamics, control inputs, and candidate barrier functions are polynomial. 
In addition, finding a valid CBF typically requires careful selection of an appropriate monomial parameterization basis for the barrier~function. %

To overcome these challenges, CBFs represented as neural networks—\emph{neural CBFs}—have become a popular alternative, enabling the \emph{inductive synthesis} of barrier functions for a broad class of dynamics and safety constraints~\citep{zhao_synthesizing_2020,zhang_exact_nodate,hu_verification_2024,abate_formal_2021}. Neural CBFs allow for arbitrary nonlinear dynamics, and the expressiveness of the barrier function can be increased with the neural network's size, thereby expanding the obtainable control invariant set. However, the main drawback of learning a \emph{candidate barrier~function} arguably lies in the high cost of post-hoc verification required to prove that the candidate is a valid CBF. This verification stage constitutes the primary computational bottleneck, constrained by two key factors: (1) the size of the neural network and (2) the dimensionality of the dynamical system. Existing verification approaches that rely on satisfiability modulo theories (SMT) solvers \citep{zhao_synthesizing_2020,abate_fossil_2021,sha_synthesizing_2021, Edwards_2024} or mixed-integer programming (MIP)~\citep{DBLP:conf/hybrid/Zhao0ZZTL22} exhibit poor scalability with respect to network size. Approaches that have employed branch-and-bound verification have until now been restricted to ReLU-based neural CBFs, as this facilitates efficient computation of the gradient of the network \citep{shi2022efficiently, zhang_exact_nodate, hu_verification_2024}. Learning neural CBFs for challenging problems requires introducing sufficient nonlineararity into the network, which motivates our focus on networks with arbitrary nonlinear activation functions.
\paragraph{Our approach}
In this paper, we propose a scalable verification technique for neural CBFs that overcomes the limitations of SMT-based verifiers while still allowing for arbitrary activation~functions.
Instead of using SMT solvers to reason over the exact nonlinear conditions required for a function to be a CBF, we derive sufficient linear conditions for more efficient verification of a neural candidate barrier function.
To compute these sufficient linear conditions, we extend \emph{linear bound propagation} (LBP)~\citep{zhang_efficient_2018} to derive linear upper and lower bounds on the gradients of the network. 
Rooted in recent advances in neural network verification \citep{zhang_efficient_2018, Shi2025}, LBP has been applied to the verification of \emph{discrete-time} stochastic barrier functions \citep{mathiesen2022safety}. 
Our setting with \emph{continuous-time} dynamics requires a derivative condition to establish~set invariance, necessitating linear bounds to the network gradients and linear relaxations of Lie derivatives.

We combine our techniques for computing bounds on the network gradients with McCormick relaxation~\citep{mccormick_computability_1976} and certified bounds on the dynamics to derive linear upper and lower bounds on the CBF conditions.
Importantly, our approach supports the inclusion of control variables in the invariance condition, making our approach attractive for use with downstream applications such as safety filters~\citep{Wabersich2023}.
Finally, our verification approach is agnostic to the neural network training and can (while beyond our scope) be integrated into common learner-verifier frameworks for synthesizing neural CBFs~\citep{peruffo_automated_2020,dawson_safe_2022}.

Overall, our key contributions are summarized as~follows: 
\begin{enumerate}
    \item We introduce a novel method to verify candidate neural barrier functions via LBP and McCormick relaxation, thereby eliminating the need for computationally expensive verification procedures and enabling the verification of larger networks.
    \item We develop a tailored and parallelizable refinement strategy that adaptively refines a simplicial mesh over the state space, thus reducing conservatism in the lower and upper bounds.
    \item We demonstrate that our approach enables the verification of neural CBFs with larger networks than state-of-the-art SMT-based verification techniques.
\end{enumerate}

We begin by providing background on neural CBFs for safety certification %
in~\cref{sec:problem_setup}. We~introduce our contributions on extending LBP in~\cref{sec:LBP} and the verification procedure for the CBF in~\cref{sec:verification_procedure}. Finally, we experimentally demonstrate the scalability of our approach to larger networks in~\cref{sec:experiments}.

\section{Problem Formulation}\label{sec:problem_setup}
We consider nonlinear control-affine dynamical systems with state space \(\statespace \subset \reals^n\), where the state \( x \in \statespace \) evolves according~to 
\begin{equation}
    \label{eq:nonlnsys}
    \dot{x}(t) = f(x(t)) + g(x(t))u(t),
\end{equation}
with continuous functions \( f \colon \statespace \to \reals^n \) and \( g \colon \statespace \to \reals^{n \times m} \), and a bounded control input \( u(t) \in \mathcal{U} = [\underline{u}, \overline{u}] \subset \reals^m \). 
The control objective is to ensure that the state \( x(t) \) remains within a desired measurable set \( \safeset \subset \statespace \) for all \( t \ge 0 \), which we refer to as the \emph{safe set}. To guarantee the system remains safe (i.e., within the safe set), we relate safety to the notion of control invariance.
\begin{definition}
    \label{def:invariance}
     A set \( \mathcal{C} \subset \reals^n \) is said to be \emph{control invariant} with respect to~\cref{eq:nonlnsys} if, for any \( x(0) \in \mathcal{C} \), there exists a measurable control signal \( u \colon [0, \infty) \to [\underline{u}, \overline{u}] \) such that \( x(t) \in \mathcal{C} \) for all \( t \ge 0 \).
\end{definition}
To find a control invariant subset of the safe set, we synthesize a \emph{control barrier function} (CBF)~\citep{DBLP:conf/eucc/AmesCENST19}. 
A continuously differentiable function \( \barrier \colon \statespace \to \reals \) is a CBF for~\cref{eq:nonlnsys} and the safe set \( \safeset \) if the superlevel set \(\mathcal{C} = \{ x \in \statespace : \barrier(x) \ge 0 \}\) satisfies \( \mathcal{C} \subseteq \safeset \), and there exists an extended class-\( \mathcal{K} \) function \( \alpha \colon \reals \to \reals \) (i.e., $\alpha$ is strictly increasing and $\alpha(0) = 0$) such that, for all \( x \in \mathcal{C} \),
\begin{equation}
    \label{eq:cbf}
    \mathcal{L}_f \barrier(x)
    + \sup_{u \in \controlspace} \left[ \mathcal{L}_g \barrier(x) u \right]
    + \alpha(\barrier(x)) \ge 0,
\end{equation}
where \( \mathcal{L}_f \barrier(x) := \nabla_x \barrier(x) f(x) \) and \( \mathcal{L}_g \barrier(x) := \nabla_x \barrier(x) g(x) \) denote the Lie derivatives of \( \barrier \) along \( f \) and \( g \), respectively.

\begin{theorem} \label{thm:cbf_conditions}[\cite{DBLP:journals/tac/AmesXGT17}]
    If $\barrier$ is a CBF for the dynamical system in \cref{eq:nonlnsys} and the safe set~$\safeset$, then the superlevel set $\mathcal{C}$ is control invariant with respect to \cref{eq:nonlnsys}.
\end{theorem}
\Cref{thm:cbf_conditions} enables different formulations of the safety verification problem.
For instance, given a set of initial states \( \mathcal{X}_0 \subset \statespace \), one may seek to prove that the system in~\cref{eq:nonlnsys} is safe with respect to \( \safeset \) and all \( x(0) \in \mathcal{X}_0 \) by finding a CBF \( \barrier \) such that \( \mathcal{X}_0 \subseteq \mathcal{C} \). Alternatively, one might aim to find a CBF whose superlevel set \( \mathcal{C} \) maximizes the volume contained in \( \safeset \).

\subsection{Neural Control Barrier Functions}\label{sec:NCBF}
Finding a CBF can be challenging, especially for nonlinear dynamics and nonconvex safe sets. As such, it has become common practice to \emph{learn} a CBF, often by training a feedforward neural network $\NN \colon \reals^n \to \reals$ with parameters $\theta$~\citep{dawson_safe_2022}.

\begin{definition}
    \label{def:NN}
    An $(L+1)$-layer neural network $\NN$ with dimensions $n_0, \ldots, n_L$ is a function \( \NN \colon \reals^{n_0} \to \reals^{n_L} \) defined as
    $
        \NN(x) = (\layer^{(\numlayers)} \funccomp \layer^{(\numlayers - 1)} \funccomp \cdots \funccomp \layer^{(2)} \funccomp \layer^{(1)})(x),
    $
    where $\layer^{(i)}(\postactivation_{i - 1}) = \activationfunc^{(i)}(\nnweight^{(i)}\postactivation_{i - 1} + \nnbias^{(i)}) = \postactivation_i$ is the $i^\text{th}$ layer with parameters $\nnweight^{(i)} \in \reals^{n_i \times n_{i-1}}$ and $\nnbias^{(i)} \in \reals^{n_i}$ for $\postactivation_{i - 1} \in \reals^{n_{i-1}}$ and $\postactivation_i \in \reals^{n_i}$, and nonlinear activation function $\activationfunc^{(i)}(\cdot):\reals \rightarrow\reals$.
\end{definition}

Applying $\activationfunc^{(i)}(\cdot)$ to a vector $\postactivation_{i} \in \reals^{n_i}$ is understood as an elementwise application of the scalar activation function $\activationfunc^{(i)}$. For each layer $i$, define the \emph{pre-activation output} as $\preactivation_i = \nnweight^{(i)}\postactivation_{i - 1} + \nnbias^{(i)}$ and the \emph{post-activation output} as $\postactivation_i = \activationfunc^{(i)}(\preactivation_i)$.
We assume the final layer does not have an activation function, i.e., $\activationfunc^{(L)}(y_L)=y_L$.
Thus, the final output is $\NN(x) = \postactivation_\numlayers = \preactivation_\numlayers$.
To allow deriving linear bounds on the activation function and its derivative in \cref{sec:LBP}, we make the following~assumption.
\begin{assumption}
    We assume that the activation function is semidifferentiable. %
\end{assumption}

The candidate barrier function can be trained on by minimizing a loss function that resembles a differentiable version of the CBF conditions. Various approaches have been developed to guide the training, many of which are rooted in~\emph{counterexample-guided inductive synthesis} (CEGIS) techniques~\citep{DBLP:conf/cav/AbateDKKP18,dawson_safe_2022}.

\subsection{Problem Statement: Certification of a Valid Control Barrier Function}\label{subsec:verification_CBF}
When working with neural CBFs, a key challenge is to verify that a trained neural network is indeed a valid CBF.
In this paper, we focus on this verification problem, which is formalized as follows.
\begin{problem}
    \label{prob:verify}
    Given the dynamics in \cref{eq:nonlnsys}, a safe set $\safeset$, and a candidate neural CBF \( \NN \), verify that \(\NN\) is a valid CBF, i.e., the superlevel set $\mathcal{C}$ of \(\NN\) is a control invariant subset of the safe set $\safeset$.
\end{problem}

To solve Problem~\ref{prob:verify}, we need to establish that $\NN$ satisfies $\mathcal{C} = \{ x \in \statespace : \NN(x) \geq 0 \} \subseteq \safeset$ and the condition in~\cref{eq:cbf} is satisfied, which can be encoded by the following formula:
\begin{equation}\begin{split} \label{eq:phi}
    \phi \coloneq {}&{} \forall x \in \mathcal{S} :\left[
        \NN(x) < 0 \lor
            \frac{\partial \NN(x)}{\partial x}f(x) + \sup_{u\in\mathcal{U}} \frac{\partial \NN(x)}{\partial x}g(x)u + \alpha \NN(x) \ge 0 
    \right] \\
    &{} \land \: \forall x \in \mathcal{S}^C : \left[ \NN(x) < 0 \right],
\end{split}\end{equation}
where $\safeset^C = \statespace \setminus \safeset$ is the complement of $\safeset$. 
The satisfiability of the formula \(\phi\) implies that the candidate barrier function $\NN$ is a valid CBF. Previous works~\citep{peruffo_automated_2020, abate_fossil_2021, sha_synthesizing_2021, zhao_synthesizing_2020, abate_formal_2021} have used SMT solvers for quantifier-free nonlinear real arithmetic to search for satisfying assignments of the negation of \(\phi\). However, using SMT solvers introduces a significant computational bottleneck, thus limiting both the size of the neural network and its input dimension (corresponding to the state space dimension). 

\paragraph{Overview of our approach}
To enable the use of larger neural networks, we circumvent reasoning over quantifier-free nonlinear real arithmetic and instead consider linear over- and under-approximations of the nonlinearities in the formula \(\phi\).
To obtain these over- and under-approximations over a domain of interest $\Delta \subset \statespace$, we will bound the value $\NN(x)$ and gradients $\frac{\partial \NN(x)}{\partial x}$ of the network (discussed in \cref{sec:LBP}), and the dynamics $f$ and $g$ (discussed in \cref{subsec:taylor}).
In \cref{subsec:linear_cbf_condition}, we combine the resulting linear bounds into a (conservative) linear surrogate \(\phi_{\mathrm{linear}}\) for $\phi$.
In particular, satisfiability of the surrogate \(\phi_{\mathrm{linear}}\) implies satisfiability of \(\phi\) and thus proves that $\NN$ is a valid CBF.
Finally, in \cref{subsec:certificate_refinement} we present our refinement strategy crucial to reducing conservatism of \(\phi_{\mathrm{linear}}\).

\section{Linear Bound Propagation for Neural CBF Verification}\label{sec:LBP}
In this section, we describe how to compute linear upper and lower bounds on \(\NN(x)\) and its gradients \(\frac{\partial \NN(x)}{\partial x}\) over a fixed domain of interest, denoted as \(\convex \subset \statespace\).
We compute the bounds on \(\NN(x)\) using existing linear bound propagation (LBP) techniques from~\citet{zhang_efficient_2018}.
Construction of the bounds on \(\NN(x)\) is discussed in \cref{appendix:LBP} and yields
\begin{equation}
    \lowerbound{\NN}(x) \coloneqq \lowerbound{A}_{\NN} x + \lowerbound{a}_{\NN} \le \NN(x) \le \upperbound{A}_{\NN} x + \upperbound{a}_{\NN} \eqqcolon \upperbound{\NN}(x), 
    \quad \forall x \in \convex.
\end{equation}

By contrast, the bounds on the gradient \( \frac{\partial \NN(x)}{\partial x}\) require further relaxation of bilinear terms and the derivatives of the activation function, which cannot be computed using standard LBP. 
In this section, we present our first key contribution, which is an extension of LBP that incorporates these further relaxations to bound the gradient. We draw inspiration from~\cite{wicker2023robust}, which obtains gradient bounds via a backward pass with interval bounds. 
In contrast, we use LBP to obtain tighter bounds and, as in~\cite{eiras_efficient_2024}, apply McCormick relaxations to the bilinear terms. 
As a novel feature, we formulate bounds with respect to pre-activation outputs and relax the bilinearity induced by the composition of successive layers.
Notably, this approach eliminates the need for an additional forward pass of the neural network, as required in, e.g.,~\cite{eiras_efficient_2024}.
Our approach yields the following bounds on the gradient, with their construction being given in the proof of \cref{thm:linear_derivative_bounds}.

\begin{theorem}[Linear bounds on Jacobian]
\label{thm:linear_derivative_bounds}
    Given an \((L+1)\)-layer neural network \( \NN(x) \) as in Def.~\ref{def:NN} and a domain $\convex \subset \statespace$, we can construct linear bounds on $\frac{\partial \NN(x)}{\partial x}$ for all $x \in \convex$ of the form
    \begin{equation}
        \label{eq:linear_derivative_bounds}
        \lowerbound{\partial\NN}(x) \coloneqq \lowerbound{\Pi} x + \lowerbound{\pi} \leq \frac{\partial \NN(x)}{\partial x} \leq \upperbound{\Pi} x + \upperbound{\pi} \eqqcolon \upperbound{\partial\NN}(x),
    \end{equation}
    where the inequalities are understood elementwise and the coefficients 
    \((\lowerbound{\Pi}, \lowerbound{\pi}, \upperbound{\Pi}, \upperbound{\pi})\) 
    are computed recursively through affine relaxations of the layer Jacobians.
\end{theorem}
\begin{proof}
We can write the gradient via the chain rule as the product of the Jacobian of each layer:
\begin{equation}
    \label{eq:Jacobians_product}
    \frac{\partial \NN(x)}{\partial x} = \gradient_{\postactivation_{\numlayers - 1}} \layer^{(\numlayers)}(\postactivation_{\numlayers - 1}) \,
    \gradient_{\postactivation_{\numlayers - 2}} \layer^{(\numlayers - 1)}(\postactivation_{\numlayers - 2}) \cdots
    \gradient_{\postactivation_{1}} \layer^{(2)}(\postactivation_{1}) \,
    \gradient_{\postactivation_{0}} \layer^{(1)}(\postactivation_{0}).
\end{equation}
The Jacobian \(\mathcal{J}^{(i)} \in \reals^{n_i \times n_{i-1}}\) of layer \(i\) is defined by the derivative of the activation~function \({\activationfunc^{(i)}}'\) and the weight matrix $\nnweight^{(i)}$:
\begin{equation*}
    \mathcal{J}^{(i)}  \coloneqq \gradient_{\postactivation_{i-1}} \layer^{(i)}(\postactivation_{i-1})
    = \frac{\partial \postactivation_i}{\partial \preactivation_i}(\preactivation_i)
      \frac{\partial \preactivation_i}{\partial \postactivation_{i-1}}
    = \diag \big( {\activationfunc^{(i)}}'(\preactivation_{i}) \big) \nnweight^{(i)},
\end{equation*}
where $\diag(x)$ is the diagonal matrix for the vector $x$.
We recursively bound the product of the Jacobians in \cref{eq:Jacobians_product}.
First, we construct linear bounds on the Jacobian of the final layer. %
Then, the Jacobian of each previous layer \(i-1\) is combined with the current bounds on the product of Jacobians from layers $i$ to $L$. 
Thus, we seek to derive the following sequence of~bounds:
\begin{align*}
    \lowerbound{\Pi}^{(L)} \preactivation_{L} + \lowerbound{\pi}^{(L)} 
    & \leq \mathcal{J}^{(L)}(\preactivation_L) \leq \upperbound{\Pi}^{(L)} \preactivation_{L} + \upperbound{\pi}^{(L)}, \\
    \lowerbound{\Pi}^{(i)} \preactivation_{i} + \lowerbound{\pi}^{(i)} 
    &\leq 
    \mathcal{J}^{(i+1)}(\preactivation_{i+1}) 
    \mathcal{J}^{(i)}(\preactivation_{i}) \leq \upperbound{\Pi}^{(i)} \preactivation_{i} + \upperbound{\pi}^{(i)},
    \quad i = L-1, \dots, 1.
\end{align*}
We derive the linear bounds on \(\mathcal{J}^{(i)}(\preactivation_i)\) for $i=1,\ldots,L$ in \cref{appendix:LBP} and consider bounding the product of the Jacobians for two layers $i$ and $i+1$ in the remainder of this proof. This requires writing \(\mathcal{J}^{(i+1)}\) in terms of \(\preactivation_i\). To do so, we substitute \(\preactivation_{i+1}\) with \(\nnweight^{(i+1)}\activationfunc^{(i)}(\preactivation_i) + \nnbias^{(i+1)}\) and use the linear bounds on \(\activationfunc^{(i)}(\preactivation_i)\) computed during the forward pass of standard~LBP. 
The resulting linear lower and upper bounds of \(\mathcal{J}^{(i+1)}(\preactivation_i)\) are given in \cref{eq:J_bnds} of \cref{appendix:LBP}. Using Einstein notation, an element \((j,k)\) of the resulting product is a sum over the index \(p\): \((\mathcal{J}^{(i+1)}\mathcal{J}^{(i)})_{jk} = (\mathcal{J}^{(i+1)})_{jp} (\mathcal{J}^{(i)})_{pk}\). %
We use McCormick relaxations~\citep{mccormick_computability_1976} for bilinear terms to~obtain
\begin{align*}
    (\mathcal{J}^{(i+1)})_{jp}(\mathcal{J}^{(i)})_{pk}
    &\geq (\lowerbound{\mathcal{J}}^{(i+1)})_{jp}(\mathcal{J}^{(i)})_{pk}
     + (\mathcal{J}^{(i+1)})_{jp}(\lowerbound{\mathcal{J}}^{(i)})_{pk}
     - (\lowerbound{\mathcal{J}}^{(i+1)})_{jp}(\lowerbound{\mathcal{J}}^{(i)})_{pk},\\
    (\mathcal{J}^{(i+1)})_{jp}(\mathcal{J}^{(i)})_{pk}
    &\geq (\upperbound{\mathcal{J}}^{(i+1)})_{jp}(\mathcal{J}^{(i)})_{pk}
     + (\mathcal{J}^{(i+1)})_{jp}(\upperbound{\mathcal{J}}^{(i)})_{pk}
     - (\upperbound{\mathcal{J}}^{(i+1)})_{jp}(\upperbound{\mathcal{J}}^{(i)})_{pk},
\end{align*}
where \((\lowerbound{\mathcal{J}}^{(i)})_{pk}\) and \((\upperbound{\mathcal{J}}^{(i)})_{pk}\) denote interval bounds for \((\mathcal{J}^{(i)})_{pk}\) over \(\convex\). Since either of the two lower bounds is valid individually, we consider the convex combination with parameter \(\eta_{jpk} \in [0, 1]\):
\begin{align*}
    (\mathcal{J}^{(i+1)})_{jp}(\mathcal{J}^{(i)})_{pk} \geq &\;
    \eta_{jpk} \Big( (\underline{\mathcal{J}}^{(i+1)})_{jp}(\mathcal{J}^{(i)})_{pk}
    + (\mathcal{J}^{(i+1)})_{jp}(\underline{\mathcal{J}}^{(i)})_{pk}
    - (\underline{\mathcal{J}}^{(i+1)})_{jp}(\underline{\mathcal{J}}^{(i)})_{pk} \Big) \nonumber \\
    &\!\!\!\!\!\!\!\!\!\! + (1 - \eta_{jpk}) \Big( (\overline{\mathcal{J}}^{(i+1)})_{jp}(\mathcal{J}^{(i)})_{pk}
    + (\mathcal{J}^{(i+1)})_{jp}(\overline{\mathcal{J}}^{(i)})_{pk}
    - (\overline{\mathcal{J}}^{(i+1)})_{jp}(\overline{\mathcal{J}}^{(i)})_{pk} \Big).
\end{align*}
Let \( G^+ = \max(G, 0) \) and \( G^- = \min(G, 0) \) be the positive and negative parts of a matrix $G$.
Substituting the affine bounds for \((\mathcal{J}^{(i)})_{pk}\) and \((\mathcal{J}^{(i+1)})_{jp}\) yields
\begin{align} \label{eq:mccormick}
    (\mathcal{J}^{(i+1)})_{jp}(\mathcal{J}^{(i)})_{pk} \geq &\;
    \left(\eta_{jpk}(\underline{\mathcal{J}}^{(i+1)})_{jp} + (1 - \eta_{jpk})(\overline{\mathcal{J}}^{(i+1)})_{jp}\right)^+
    \left((\underline{\Lambda}^{(i)})_{pkm}y_{i,m} + (\underline{\lambda}^{(i)})_{pk}\right) \nonumber \\
    &\; + \left(\eta_{jpk}(\underline{\mathcal{J}}^{(i+1)})_{jp} + (1 - \eta_{jpk})(\overline{\mathcal{J}}^{(i+1)})_{jp}\right)^-
    \left((\overline{\Lambda}^{(i)})_{pkm}y_{i,m} + (\overline{\lambda}^{(i)})_{pk}\right) \nonumber \\
    &\; + \left((\underline{\Pi}^{(i+1)})_{jpm}y_{i,m} + (\underline{\pi}^{(i+1)})_{jp}\right)
    \left(\eta_{jpk}(\underline{\mathcal{J}}^{(i)})_{pk} + (1 - \eta_{jpk})(\overline{\mathcal{J}}^{(i)})_{pk}\right)^+ \nonumber \\
    &\; + \left((\overline{\Pi}^{(i+1)})_{jpm}y_{i,m} + (\overline{\pi}^{(i+1)})_{jp}\right)
    \left(\eta_{jpk}(\underline{\mathcal{J}}^{(i)})_{pk} + (1 - \eta_{jpk})(\overline{\mathcal{J}}^{(i)})_{pk}\right)^- \nonumber \\
    &\; - \eta_{jpk}(\underline{\mathcal{J}}^{(i+1)})_{jp}(\underline{\mathcal{J}}^{(i)})_{pk}
    - (1 - \eta_{jpk})(\overline{\mathcal{J}}^{(i+1)})_{jp}(\overline{\mathcal{J}}^{(i)})_{pk}.
\end{align}
To obtain the overall bounds on \(\frac{\partial \NN(x)}{\partial x}\), we first bound the product of \(\mathcal{J}^{(L)} \mathcal{J}^{(L-1)}\) and then recursively obtain bounds on the product \(\prod_{j=i}^L \mathcal{J}^{(j)}\) for $i=L-2, L-3, \ldots, 1$ by applying \cref{eq:mccormick} to bound the product of \(\mathcal{J}^{(i)}\) and \(\prod_{j=i+1}^L \mathcal{J}^{(j)}\). Repeating this process yields the bounds in \cref{eq:linear_derivative_bounds}.
\end{proof}

The computation of the linear bounds can be efficiently batched, which enables efficient and parallelized use of the GPU for computing bounds on \(\frac{\partial \NN(x)}{\partial x}\) over multiple regions \(\convex\) simultaneously.

\section{Verification Procedure}\label{sec:verification_procedure}
Having bounded the values and gradients of $\NN$, we now turn to bounding the dynamics in~\cref{eq:nonlnsys} and combining all bounds into a conservative but linear surrogate for the formula $\phi$ in \cref{eq:phi}.

\subsection{First-Order Model of the System Dynamics}\label{subsec:taylor}
We derive upper and lower bounds on the nonlinear system dynamics \( f(x) \) and \( g(x) \) by certified first-order Taylor expansions over the domain $\convex \subset \statespace$, which are defined as follows.
\begin{proposition}[Certified first-order Taylor expansion]\label{prop:cert_taylor_expansion}
Let \( f : \statespace \to \reals^n \) and \( g : \statespace \to \reals^{n \times m} \) be continuously differentiable functions, and let \( \convex \subset \statespace \) be a convex domain. Then, there exist hyperrectangles \( \mathcal{R}_f \subseteq \reals^n \) and \( \mathcal{R}_g \subseteq \reals^{n \times m} \) such that for all \( x \in \convex \),
\[
    f(x) \in \left( f(c) + \nabla_x f(c) (x - c) \right) \oplus \mathcal{R}_f, \qquad
    g(x) \in \left( g(c) + \nabla_x g(c) (x - c) \right) \oplus \mathcal{R}_g,
\]
where \( \oplus \) denotes the Minkowski sum and \( c \in \convex \) is the expansion point.
\end{proposition}
Using the Lagrange error bound, the remainder terms \(\mathcal{R}_f\) and \(\mathcal{R}_g\) can be efficiently computed when \( f \) and \(g\) are twice continuously differentiable (see \cref{appendix:taylor}). When \(f\) or \(g\) are not given in elementary form, linear relaxations can be computed from the (local) Lipschitz constant. In the case that \(f\) or \(g\) are represented by neural networks, linear bounds can be obtained efficiently using LBP, analogous to the computation of bounds on \( \NN(x) \). We denote the linear bounds on $f(x)$, $x \in \convex$, as
\begin{align*}
    \lowerbound{f}(x) \coloneqq A_f x + b_f - \lowerbound{r}_f, 
    \qquad\qquad
    \upperbound{f}(x) \coloneqq A_f x + b_f + \upperbound{r}_f,
\end{align*}
where \(\lowerbound{r}_f, \upperbound{r}_f\) are the lower and upper bounds on \(\mathcal{R}_f\). The bounds on \(g(x)\) are denoted analogously. %

\subsection{Linear Bounds on the CBF Constraint}\label{subsec:linear_cbf_condition}
Having derived certified linear bounds for all terms in the CBF conditions, we now construct the linear surrogate of \cref{eq:phi} to solve Problem~\ref{prob:verify}. We denote bounds over the domain of interest \(\convex\) as
\begin{align}
    \label{eq:all_bounds}
    \mathcal{B}_{\min} &\leq \min_{x \in \convex} \NN(x) \leq \mathcal{B}_{\max}, 
    &\qquad
    \partial\mathcal{B}_{\min} &\leq \min_{x \in \convex} \frac{\partial \NN(x)}{\partial x} \leq \partial\mathcal{B}_{\max}, \nonumber
    \\
    \mathcal{F}_{\min} &\leq \min_{x \in \convex} f(x) \leq \mathcal{F}_{\max},
    &\qquad
    \mathcal{G}_{\min} &\leq \min_{x \in \convex} g(x) \leq \mathcal{G}_{\max}.
\end{align}
which can be easily computed from the linear bounds already derived in \cref{sec:LBP,subsec:taylor}. 
As in \cref{sec:LBP}, we define \( G^+ = \max(G, 0) \) and \( G^- = \min(G, 0) \) for the matrix $G$.
We utilize McCormick relaxations to bound the drift term $\frac{\partial \NN(x)}{\partial x}f(x)$ based on the bounds in \cref{eq:all_bounds}:
\begin{equation}
\begin{split}
    \frac{\partial \NN(x)}{\partial x}f(x) \geq {}&{} (C_{\partial\mathcal{B}}(\eta))^+ \lowerbound{f}(x) + (C_{\partial\mathcal{B}}(\eta))^- \upperbound{f}(x) + (C_{\mathcal{F}}(\eta))^+ \lowerbound{\partial\NN}(x) \\
    & + (C_{\mathcal{F}}(\eta))^- \upperbound{\partial\NN}(x) - \big(\eta \partial\mathcal{B}_{\min} \mathcal{F}_{\min} + (1-\eta)\partial\mathcal{B}_{\max} \mathcal{F}_{\max}\big),
    \label{eq:bound_drift}
\end{split}
\end{equation}
where \(\eta \in [0,1]^n\) defines the convex combination of the McCormick inequalities, via \(C_{\partial\mathcal{B}}(\eta) \coloneqq \eta \partial\mathcal{B}_{\min} + (1-\eta)\partial\mathcal{B}_{\max} \), and \(C_{\mathcal{F}}(\eta) \coloneqq \eta \mathcal{F}_{\min} + (1-\eta)\mathcal{F}_{\max}\). 
We rewrite \cref{eq:bound_drift} in affine form~as
\begin{align}
    \frac{\partial \NN(x)}{\partial x} f(x)
    &\ge \Gamma_{\mathrm{drift, L}}(\eta)\, x +  \beta_{\mathrm{drift, L}}(\eta),
\end{align}
with the coefficients and constant terms defined as
\begin{align}
    \Gamma_{\mathrm{drift, L}}(\eta)
    &\coloneqq C_{\partial\mathcal{B}}(\eta)\, A_f
    + (C_{\mathcal{F}}(\eta))^{+}\, \underline{\Pi}
    + (C_{\mathcal{F}}(\eta))^{-}\, \overline{\Pi}, \\
    \beta_{\mathrm{drift, L}}(\eta)
    &\coloneqq C_{\partial\mathcal{B}}(\eta)\, b_f
    - (C_{\partial\mathcal{B}}(\eta))^{+}\, \underline{r}_f
    + (C_{\partial\mathcal{B}}(\eta))^{-}\, \overline{r}_f
    + (C_{\mathcal{F}}(\eta))^{+}\, \underline{\pi} \nonumber \\
    &\quad
    + (C_{\mathcal{F}}(\eta))^{-}\, \overline{\pi} - \big(\eta\, \partial\mathcal{B}_{\min}\, \mathcal{F}_{\min}
      + (1-\eta)\, \partial\mathcal{B}_{\max}\, \mathcal{F}_{\max}\big).
\end{align}
Next, to bound the control term, recall from \cref{sec:problem_setup} that $\mathcal{U} = [\underline{u}_j, \overline{u}_j]$, such that
\(
    \sup_{u \in \mathcal{U}} J(x)g(x)u = \sum_{j=1}^{m} \max_{u_j \in [\underline{u}_j, \overline{u}_j]} \Big( \sum_{p=1}^{n} J_p(x) g_{pj}(x) \Big) u_j. 
\)
As with the drift term, we can derive linear bounds for the term \(\sum_{p=1}^{n} J_p(x) g_{pj}(x)\), denoted as \(\underline{v}_j(x)\) and \(\overline{v}_j(x)\), so that we write the control term as:
\begin{align}
    \sup_{u \in \mathcal{U}} \frac{\partial \NN(x)}{\partial x}g(x)u \ge
    \sum_{j=1}^{m} \max\big(\underline{v}_j(x)\underline{u}_j,\, \underline{v}_j(x)\overline{u}_j\big) 
    \eqqcolon \Gamma_{\mathrm{ctrl, L}}(\eta)\, x + \beta_{\mathrm{ctrl, L}}(\eta). 
\end{align}
Finally, we combine the bounds on the control and drift term to derive a linear surrogate $\phi_{\mathrm{linear}}$ of the formula $\phi$ defined in \cref{eq:phi}:
\begin{align} \label{eq:phi_nonlinear}
    \phi_{\mathrm{linear}} \coloneqq {}&{} \forall x \in \mathcal{S} :\left[\,
        \upperbound{\NN}(x) < 0 \lor
           \psi_{\mathrm{invar}}(x) 
    \right] \land \: \forall x \in \mathcal{S}^C : \left[\, \upperbound{\NN}(x) < 0 \right],
\end{align}
where $\psi_{\mathrm{invar}}(x)$ is a SAT formula that represents the CBF invariance condition in \cref{eq:cbf}:
\begin{align} \label{eq:psi}
    \psi_{\mathrm{invar}}(x) \coloneqq \left[\Gamma_{\mathrm{drift, L}}(\eta) + \Gamma_{\mathrm{ctrl, L}}(\eta) + \alpha \lowerbound{A}_{\NN}\,\right] x + \beta_{\mathrm{drift, L}}(\eta) + \beta_{\mathrm{ctrl, L}}(\eta) + \alpha \lowerbound{a}_{\NN} \geq 0.
\end{align}
The formula $\phi_{\mathrm{linear}}$ can be used as a conservative surrogate of the original formula $\phi$, i.e., satisfaction of $\phi_{\mathrm{linear}}$ implies satisfaction of $\phi$.
This yields the next theorem, for which the proof is in \cref{appendix:proof}:
\begin{theorem} \label{thm:phi_linear}
    If \(\phi_{\mathrm{linear}}\) is satisfied, then \(\phi\) in \cref{eq:phi} is satisfied and \(\NN(x)\) is a valid CBF.
\end{theorem}

\paragraph{Counterexamples}
We extend our framework to search for counterexamples of $\phi$, i.e., satisfying assignments to the negation of \(\phi\) in \cref{eq:phi}.
We define the linear surrogate for the negation of \(\phi\) as
\begin{align} \label{eq:phi_nonlinear_neg}
    \phi_{\mathrm{linear}}^{\mathrm{NEG}} \coloneqq {}&{} \exists x \in \mathcal{S} :\left[\,
        \lowerbound{\NN}(x) \ge 0 \land
           \psi_{\mathrm{invar}}^{\mathrm{NEG}}(x) 
    \right] \lor \: \exists x \in \mathcal{S}^C : \left[\, \lowerbound{\NN}(x) \ge 0 \right],
\end{align}
where the SAT formula $\psi_{\mathrm{invar}}^{\mathrm{NEG}}(x)$ is defined as
\begin{align} \label{eq:linear_cbf_upper}
    \psi_{\mathrm{invar}}^{\mathrm{NEG}}(x) \coloneqq \left[\Gamma_{\mathrm{drift, U}}(\eta) + \Gamma_{\mathrm{ctrl, U}}(\eta) + \alpha \upperbound{A}_{\NN}\,\right] x + \beta_{\mathrm{drift, U}}(\eta) + \beta_{\mathrm{ctrl, U}}(\eta) + \alpha \upperbound{a}_{\NN} < 0.
\end{align}
Note that \cref{eq:phi_nonlinear_neg,eq:linear_cbf_upper} use the opposite bounds on each term compared to \cref{eq:phi_nonlinear,eq:psi}.
Hence, any satisfying assignment of \(x\) to \(\phi_{\mathrm{linear}}^{\mathrm{NEG}}\) is also a satisfying assignment to the negation of \(\phi\) and thus proves that $\NN$ is \emph{not} a valid CBF.
This ability to provide counterexamples enables the use of our approach in common counterexample-guided learner-verifier frameworks~\citep{peruffo_automated_2020}.

\subsection{Refinement of the Domains for Linear Bounds}\label{subsec:certificate_refinement}
Linear upper and lower bounds are generally overly conservative over large input domains. 
As a result, for a given domain \(\convex\), the linear surrogates \(\phi_{\mathrm{linear}}\) and \(\phi_{\mathrm{linear}}^{\mathrm{NEG}}\) might both be unsatisfiable, i.e., the bounds are too loose to obtain a conclusive verification outcome.
A standard way to mitigate this conservatism is to employ an adaptive refinement strategy that splits regions in which both \(\phi_{\mathrm{linear}}\) and \(\phi_{\mathrm{linear}}^{\mathrm{NEG}}\) are not satisfied~\citep{DBLP:journals/jacm/ClarkeGJLV03,DBLP:conf/formats/DierksKL07,DBLP:conf/hybrid/TiwariK02}.
Refinements with hyperrectangular regions are especially common due to their simplicity~\citep{mathiesen_certified_2025,DBLP:conf/aaai/ZikelicLHC23,DBLP:journals/fmsd/JungesAHJKQV24,DBLP:conf/cav/BadingsKJJ25}.
Closest to our setting is~\cite{mathiesen_certified_2025}, which prioritizes refinements in the dimension that contributes most to the Lagrange error bounds of the Taylor expansion used for linearizing the dynamics.

When certifying CBFs, however, identifying suitable refinement directions is substantially more difficult. 
The CBF invariance condition in~\cref{eq:cbf} generally couples multiple input dimensions in a nonlinear manner, reducing the effectiveness of refinements based on hyperrectangular regions.
To overcome this problem, we instead adopt a \emph{simplicial mesh} over the input domain, i.e., a partition into simplices, and refine regions by splitting them along their longest edge.

\begin{definition}[Simplex~\citep{crane2018discrete}]   
A simplex, \(\simplex\), is the convex hull in \(\reals^n\) of \(n+1\) affinely independent points \(v = \{v_0, \ldots, v_n\}\), called its vertices: \(\simplex \coloneqq \operatorname{conv}\{v_0, \dots, v_n\} \subset \reals^n.\)
\end{definition}

Initially, we partition the input domain \(\statespace\) into a set of simplices \( \{\simplex^1, \ldots, \simplex^q\} \) such that \(\statespace \subseteq \bigcup_{i=1}^q \simplex^i\). 
For each simplex $\simplex^i$, we perform the next steps.
First, we compute linear upper and lower bounds on \(\NN(x)\) and evaluate these affine bounds at the vertices to obtain the maximal and minimal attainable values of \(\NN(x)\) within that region (denoted as \(\mathcal{B}_{\max}\) and \(\mathcal{B}_{\min}\),~respectively). 
Then, we sequentially check for the satisfiability of the linear surrogate formulae $\phi_{\mathrm{linear}}$ and $\phi_{\mathrm{linear}}^{\mathrm{NEG}}$ defined in~\cref{eq:phi_nonlinear,eq:phi_nonlinear_neg}, as shown in \cref{fig:flowchart}.
If the verification is inconclusive (i.e., both $\phi_{\mathrm{linear}}$ and $\phi_{\mathrm{linear}}^{\mathrm{NEG}}$ are unsatisfiable), we \emph{split} $\convex^i$ along its longest edge and repeat the verification for both sub-simplices.
The satisfiability of $\phi_{\mathrm{linear}}$ for \emph{all} simplices proves that $\NN$ is a valid CBF as per \cref{thm:phi_linear}.
By contrast, the satisfiability of $\phi_{\mathrm{linear}}^{\mathrm{NEG}}$ for \emph{any} simplex proves that $\NN$ is not a valid~CBF.

\paragraph{Batching simplices}
Our approach enables efficient utilization of computational resources by performing LBP on batches of simplices using GPUs. In particular, we collect simplices that require evaluation of the SAT formula \(\psi_{\mathrm{invar}}\) defined in \cref{eq:psi} into batches and compute the necessary linear bounds on the invariance condition jointly over each batch. By further parallelizing verification over simplices, potentially over multiple GPUs, our approach allows for efficient utilization of computational resources to handle more challenging verification tasks involving larger networks.

\begin{figure}[t!]
    \centering
    \includegraphics[width=0.64\linewidth]{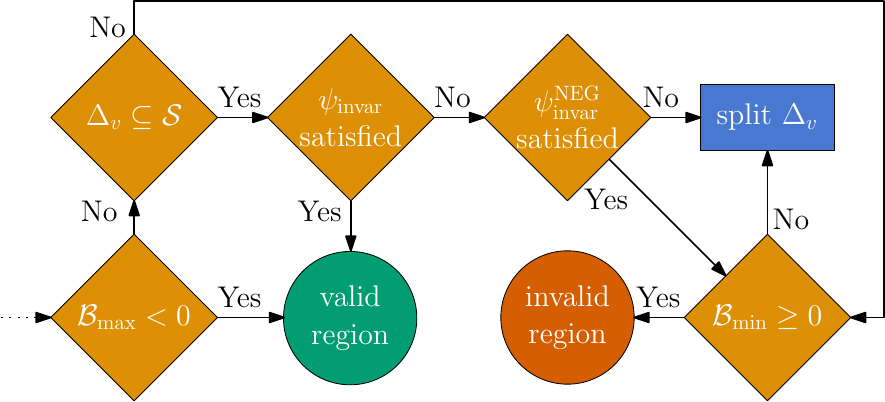}
    \caption{Verification flowchart illustrating the refinement and validation procedure for each simplex.}
    \label{fig:flowchart}
\end{figure}

\section{Experimental Results}\label{sec:experiments}
We evaluate the effectiveness of our proposed approach in scaling to larger neural network architectures. We adopt standard benchmarks from the literature~\citep{Zeng_2016, zhang_exact_nodate, abate_fossil_2021, Jiang2015, PRAJNA2006117, zhao_synthesizing_2020, Barry2012}, along with two novel benchmarks (\textit{2D Control} and \textit{Cart Pole}) chosen to demonstrate the incorporation of the \( \sup_{u} g(x)u \) term in~\cref{eq:cbf}. Additional details on all benchmarks are provided in~\cref{appendix:benchmarks}. Across all experiments, we employ the \texttt{tanh} activation function, as it consistently yields the largest control invariant sets. Implementations of the corresponding relaxations for the \texttt{sigmoid}, \texttt{ReLU}, and \texttt{leaky-ReLU} activation functions are included in the accompanying codebase\footnote{The code can be found at \url{https://github.com/Zinoex/verification-of-neural-cbf-via-lbp}.}. We fix the McCormick relaxation parameter to \(\eta = \tfrac{1}{2}\), and class-\(\mathcal{K}\) function \(\alpha = 1\) for all experiments. One could choose to optimize $\eta$, e.g., using projected gradient descent, similar to the optimization of parameterized linear relaxations in $\alpha$-CROWN \citep{xu_fast_2020}, though for simplicity we fix it to a constant. All computations ran on a machine with an Intel i7-6700K CPU, 16~GB RAM, and an NVIDIA GTX~1060 GPU (6~GB VRAM). Details regarding network training procedures are provided in~\cref{appendix:training}.

\paragraph{Results}
\cref{fig:mesh_refinement_simple2D,fig:mesh_refinement_darboux,fig:mesh_refinement_barrier2,fig:mesh_refinement_barrier3} present the verification results for the benchmarks. As shown in these figures, regions of the state space where the invariance condition (\cref{eq:psi}) or the value of $\NN$ are close to zero require the most refinement, while the mesh can remain coarse in regions where even conservative linearization is sufficient to establish the validity of the neural CBF. %

\cref{tab:verification-results} summarizes the verification performance across all benchmarks, reporting the network size, the proportion of the state space satisfying the formula $\phi$ in~\cref{eq:phi}, and the corresponding verification time. The column “number of regions” denotes the total number of regions examined during verification, including those that required subdivision. Because our networks employ \texttt{tanh} activations, a direct comparison with prior ReLU-based verification methods (e.g.,~\cite{zhang_exact_nodate}) is not meaningful. Instead, we compare against the SMT-based verification approaches of~\cite{abate_fossil_2021, Edwards_2024}, implemented via the SMT solver \texttt{dreal}~\citep{Gao2013dreal}. For a fair comparison, we use identically trained networks in both verification pipelines. As shown in \cref{tab:verification-results}, our method achieves substantial performance gains over SMT-based verification, efficiently accommodates control inputs, and scales to significantly larger neural networks.

\begin{table}[t!]
\centering
\caption{Verification results for the numerical experiments.}
\small
\begin{tabular}{ll|ccc|cc}
\toprule
& & \multicolumn{3}{c|}{Our approach} & \multicolumn{2}{c}{\texttt{dreal}} \\
\textbf{Model} & \textbf{Network} & \textbf{Time (s)} & \textbf{Certified (\%)} & \textbf{\# regions} &\textbf{Time (s)} & \textbf{Result} \\
\midrule
Barrier 2 & [64, 64] & 1.66s & 100.0 & 1432 & 635.36s & \cmark\\
Barrier 3 & [64, 64] & 5.83s & 100.0 & 3554 & 1820.33s & \cmark \\
Barrier 4 & [64, 64] & 25.37s & 100.0 & 25446 & \timeout & - \\
Darboux & [128, 256, 128] & 144.68s & 100.0 & 9142 & \timeout & - \\
2D-Control & [64, 64, 8] & 5.86s & 100.0 & 4608 & \multicolumn{2}{c}{\notapplicable} \\
Cart-Pole & [64, 64] & 2146.84s & 100.0 & 3026098 & \multicolumn{2}{c}{\notapplicable} \\
\bottomrule
\end{tabular}
\label{tab:verification-results}
\end{table}

\section{Conclusion}\label{sec:conclusion}
We have presented a scalable verification framework for neural control barrier functions that effectively handles larger neural network architectures and incorporates control inputs into the verification process. By leveraging linear bound propagation and first-order Taylor expansions, we constructed upper and lower linear bounds on the nonlinearity inherent in the verification task, thereby circumventing the need for SMT solvers capable of reasoning over quantifier-free nonlinear real arithmetic formulae. 
Our numerical experiments on a variety of benchmarks have demonstrated that our approach applies to a wide class of dynamics and to nonlinear activation functions, outperforms existing SMT-based verification tools, and scales to significantly larger neural networks. Future work will consider extending our verification framework to certificates of specifications beyond invariance, such as for reachability and avoidance.

\acks{This paper was supported by the EPSRC grant EP/Y028872/1, Mathematical Foundations of Intelligence: An “Erlangen Programme” for AI. We would like to thank Matthew Wicker for the helpful discussion and insights on linear bound propagation.}

\bibliography{references}

\newpage
\appendix

\section{Neural Network Training}\label{appendix:training}

Learning a candidate barrier function is inherently challenging and somewhat heuristic; it is not the primary focus of this work. We therefore briefly summarize the training procedure used to obtain \(\NN\), which is subsequently employed in our verification benchmarks.

The neural network \(\NN\) is trained using a two-phase curriculum learning scheme. In \emph{Phase~1}, the model learns to distinguish between safe and unsafe regions, while in \emph{Phase~2}, the CBF invariance condition is introduced to enforce forward invariance of the safe set. During training, batches of samples \(\mathcal{D}\) are periodically generated and divided into safe and unsafe subsets, \(\mathcal{D}_{S}\) and \(\mathcal{D}_{U}\), respectively. We also define \(\mathcal{D}^{p}_{U}\) as the top \(p\)-th percentile of unsafe samples with the largest \(\NN(x)\) values, which are penalized more strongly to refine the decision boundary.

Margins \(\delta_{\mathrm{safe}}\) and \(\delta_{\mathrm{unsafe}}\) encourage separation between regions, improving numerical stability. The total loss minimized during training is:
\begin{align}
    \mathcal{L}_{\mathrm{total}}
    &= \lambda_{\mathrm{safe}}\mathcal{L}_{\mathrm{safe}}
      + \lambda_{\mathrm{unsafe}}\mathcal{L}_{\mathrm{unsafe}}
      + \lambda_{\mathrm{unsafe\text{-}max}}\mathcal{L}_{\mathrm{unsafe\text{-}max}}
      + \lambda_{\mathrm{CBF}}\mathcal{L}_{\mathrm{CBF}}, \\[4pt]
    \mathcal{L}_{\mathrm{safe}}
    &= \mathbb{E}_{x\in\mathcal{D}_{S}}\!\left[\mathrm{softplus}_{\beta_{\mathrm{safe}}}\!\left(\delta_{\mathrm{safe}} - \NN(x)\right)\right], \\[4pt]
    \mathcal{L}_{\mathrm{unsafe}}
    &= \mathbb{E}_{x\in\mathcal{D}_{U}}\!\left[\mathrm{softplus}_{\beta_{\mathrm{unsafe}}}\!\left(\NN(x) + \delta_{\mathrm{unsafe}}\right)\right], \\[4pt]
    \mathcal{L}_{\mathrm{unsafe\text{-}max}}
    &= \mathbb{E}_{x\in\mathcal{D}^{p}_{U}}\!\left[\mathrm{softplus}_{\beta_{\mathrm{unsafe}}}\!\left(\NN(x) + \delta_{\mathrm{unsafe}}\right)\right], \\[4pt]
    \mathcal{L}_{\mathrm{CBF}}
    &= \mathbb{E}_{x\in\mathcal{D}_{S}}\!\left[\mathrm{softplus}_{\beta_{\mathrm{CBF}}}\!\left(-\!\left(\nabla_x \NN(x) f(x)
       + \max_{u\in\mathcal{U}}\nabla_x \NN(x) g(x)\,u
       + \alpha\,\NN(x)\right)\!\right)\right],
\end{align}
where \(\mathrm{softplus}_\beta(z)=\tfrac{1}{\beta}\log(1+e^{\beta z}) \) denotes a smooth approximation to \(\max(0,z)\) with sharpness parameter \(\beta\).
We use distinct \(\beta\) values for stability and gradient quality: \(\beta_{\mathrm{safe}} = 100, \beta_{\mathrm{unsafe}} = \beta_{\mathrm{CBF}} = 5\).
Optimization employs the AdamW optimizer with weight decay regularization. Phase~1 uses cosine annealing for smooth learning rate decay, while Phase~2 employs adaptive reduction of the learning rate upon plateau detection. The weight \(\lambda_{\mathrm{CBF}}\) is gradually increased during the transition between phases to ensure stable curriculum progression. Hyperparameters are tuned using Bayesian optimization, and the final configuration is provided in the accompanying repository.

\section{Linear Bound Propagation: Additional Remarks}\label{appendix:LBP}
To construct linear bounds on \(\NN\) over a convex domain \(\convex  \subset \statespace\), we apply the LBP procedure introduced by~\citet{zhang_efficient_2018}. For each layer \(i\) of the network, we consider linear relaxations of the activation function \(\activationfunc^{(i)}\) of the form
\begin{equation}
    \label{eq:activation_bounds}
    \lowerbound{G}_m^{(i)} \preactivation_{i,m} + \lowerbound{g}_m^{(i)} \leq \activationfunc^{(i)}(\preactivation_{i,m}) \leq \upperbound{G}_m^{(i)} \preactivation_{i,m} + \upperbound{g}_m^{(i)},
\end{equation}
where the coefficients \(\lowerbound{G}_m^{(i)}, \upperbound{G}_m^{(i)}, \lowerbound{g}_m^{(i)}, \text{and } \upperbound{g}_m^{(i)}\) are computed over the projection of~\(\convex\) to the $i^\text{th}$~layer. We provide further details on linear relaxation of the activation function in~\cref{appendix:activation_fnc}. By substituting these bounds into each layer, we bound the pre-activation output $\preactivation_{i+1}$ as a function of~$\preactivation_i$~as 
\begin{align} 
    \preactivation_{i+1}
    &\ge (W^{(i+1)})^{+}\big(\lowerbound{G}^{(i)} \preactivation_i + \lowerbound{g}^{(i)}\big)
    + (W^{(i+1)})^{-}\big(\upperbound{G}^{(i)} \preactivation_i + \upperbound{g}^{(i)}\big)
    + b^{(i+1)}, \label{eq:yi_lb} \\
    \preactivation_{i+1}
    &\le (W^{(i+1)})^{+}\big(\upperbound{G}^{(i)} \preactivation_i + \upperbound{g}^{(i)}\big)
    + (W^{(i+1)})^{-}\big(\lowerbound{G}^{(i)} \preactivation_i + \lowerbound{g}^{(i)}\big)
    + b^{(i+1)}, \label{eq:yi_ub}
\end{align}
By iterating this procedure from the input layer onward, we obtain linear bounds (with appropriate coefficients \(\lowerbound{A}_{\NN}, \upperbound{a}_{\NN}, \lowerbound{a}_{\NN}\), and \(\upperbound{a}_{\NN}\)) on the value $\NN(x)$ of the network for all $x \in \convex$ of the form
\begin{equation*}
    \lowerbound{\NN}(x) \coloneqq \lowerbound{A}_{\NN} x + \lowerbound{a}_{\NN} \le \NN(x) \le \upperbound{A}_{\NN} x + \upperbound{a}_{\NN} \eqqcolon \upperbound{\NN}(x), 
    \quad \forall x \in \convex.
\end{equation*}

For deriving the linear bounds on the gradient, we also require linear bounds on the activation derivative \({\activationfunc^{(i)}}'(\preactivation_{i,m})\) over the projection of the domain \(\convex\) onto the $i^\text{th}$ layer:
\begin{equation}
    \lowerbound{S}_m^{(i)}\preactivation_{i,m} + \lowerbound{s}_m^{(i)} \leq
    {\activationfunc^{(i)}}'(\preactivation_{i,m}) \leq
    \upperbound{S}_m^{(i)}\preactivation_{i,m} + \upperbound{s}_m^{(i)}.
\end{equation}

These can be used to bound \(\mathcal{J}^{(i)}(\preactivation_i)\). Each element \((p,k)\) of the Jacobian \(\mathcal{J}^{(i)}(\preactivation_i)\) can be expressed~as
$(\mathcal{J}^{(i)})_{pk}(\preactivation_{i}) = ({\activationfunc^{(i)}}'(\preactivation_{i}))_p  \nnweight_{pk}^{(i)}.$
Since the weight \(\nnweight_{pk}^{(i)}\) is constant, we can construct affine bounds for \((\mathcal{J}^{(i)}(\preactivation_i))_{pk}\) by multiplying the weight by the activation derivative bounds.
For the lower bound, we obtain the following.
\begin{align}\label{eq:J_bnds}
    (\lowerbound{S}^{(i)}_{p}\preactivation_{i,p} + \lowerbound{s}^{(i)}_p)(\nnweight_{pk}^{(i)})^+
     + (\upperbound{S}^{(i)}_{p}\preactivation_{i,p} + \upperbound{s}^{(i)}_p)(\nnweight_{pk}^{(i)})^-
     \leq 
     (\mathcal{J}^{(i)}(\preactivation_i))_{pk},
\end{align}
The derivation of the upper bound %
is analogous to that of the lower bound and is thus omitted for brevity. Using the bounds on $\activationfunc^{(i)}(\preactivation_i)$ from~\cref{eq:activation_bounds}, we obtain the following lower bound for \(\preactivation_{i+1}\) as a function of \(\preactivation_{i}\):
\begin{align*}
    &\lowerbound{\preactivation}_{i+1, j}(\preactivation_i)
    = (\nnweight^{(i+1)}_{jm})^+ (\lowerbound{G}^{(i)}_{m}\preactivation_{i,m} + \lowerbound{g}^{(i)}_m)
     + (\nnweight^{(i+1)}_{jm})^- (\upperbound{G}^{(i)}_{m}\preactivation_{i,m} + \upperbound{g}^{(i)}_m)
     + \nnbias^{(i+1)}_j \nonumber \\
    & \quad = \underbrace{\left( (\nnweight^{(i+1)}_{jm})^+ \lowerbound{G}^{(i)}_{m}
     + (\nnweight^{(i+1)}_{jm})^- \upperbound{G}^{(i)}_{m} \right)}_{(\lowerbound{K})_{jm}} \preactivation_{i,m}
     + \underbrace{\left( \nnbias^{(i+1)}_j
     + \sum_m (\nnweight^{(i+1)}_{jm})^+ \lowerbound{g}^{(i)}_m
     + (\nnweight^{(i+1)}_{jm})^- \upperbound{g}^{(i)}_m \right)}_{(\lowerbound{k})_{j}}.
\end{align*}
Substituting the bounds on \(\preactivation_{i+1}\) into the expression for the lower bound of \((\mathcal{J}^{(i+1)}(\preactivation_i))_{jp}\) yields:
\begin{align*}
    &(\nnweight_{jp}^{(i+1)})^+ \left(\lowerbound{S}_j^{(i+1)}\left(\sum_m (\lowerbound{K})_{jm} \preactivation_{i,m} + (\lowerbound{k})_{j}\right) + \lowerbound{s}_j^{(i+1)}\right) \nonumber \\
    &\quad + (\nnweight_{jp}^{(i+1)})^- \left(\upperbound{S}_j^{(i+1)}\left(\sum_m (\upperbound{K})_{jm} \preactivation_{i,m} + (\upperbound{k})_{j}\right) + \upperbound{s}_j^{(i+1)}\right)
     \leq 
     (\mathcal{J}^{(i)}(\preactivation_i))_{pk}.
\end{align*}
To obtain the interval bounds on \((\mathcal{J}^{(i)}(\preactivation_i))_{jp}\) over the domain \(\convex\), as used in the McCormick constraints (\cref{eq:mccormick}), we project bounds on \(\convex\) through each layer according to \cref{eq:yi_lb,eq:yi_ub}. For a simplex, this constitutes projecting each of the vertices and taking the minimum and maximum over the projected vertex.

\section{Linear Bounds on the Activation Functions}\label{appendix:activation_fnc}
For a given activation function $\sigma(y)$ and input interval $[l, u]$, we construct affine relaxations for the activation function, defined as
\begin{equation*}
    \lowerbound{G}_m \preactivation_m + \lowerbound{g}_m
    \;\leq\;
    \activationfunc(\preactivation_m)
    \;\leq\;
    \upperbound{G}_m \preactivation_m + \upperbound{g}_m,
\end{equation*}
and corresponding linear bounds on the derivative:
\begin{equation*}
    \lowerbound{S}_m \preactivation_m + \lowerbound{s}_m
    \;\leq\;
    \activationfunc'(\preactivation_m)
    \;\leq\;
    \upperbound{S}_m \preactivation_m + \upperbound{s}_m.
\end{equation*}

\subsection*{A. ReLU Relaxation}
\begin{figure}[h]
    \centering
    \includegraphics[width=0.75\linewidth]{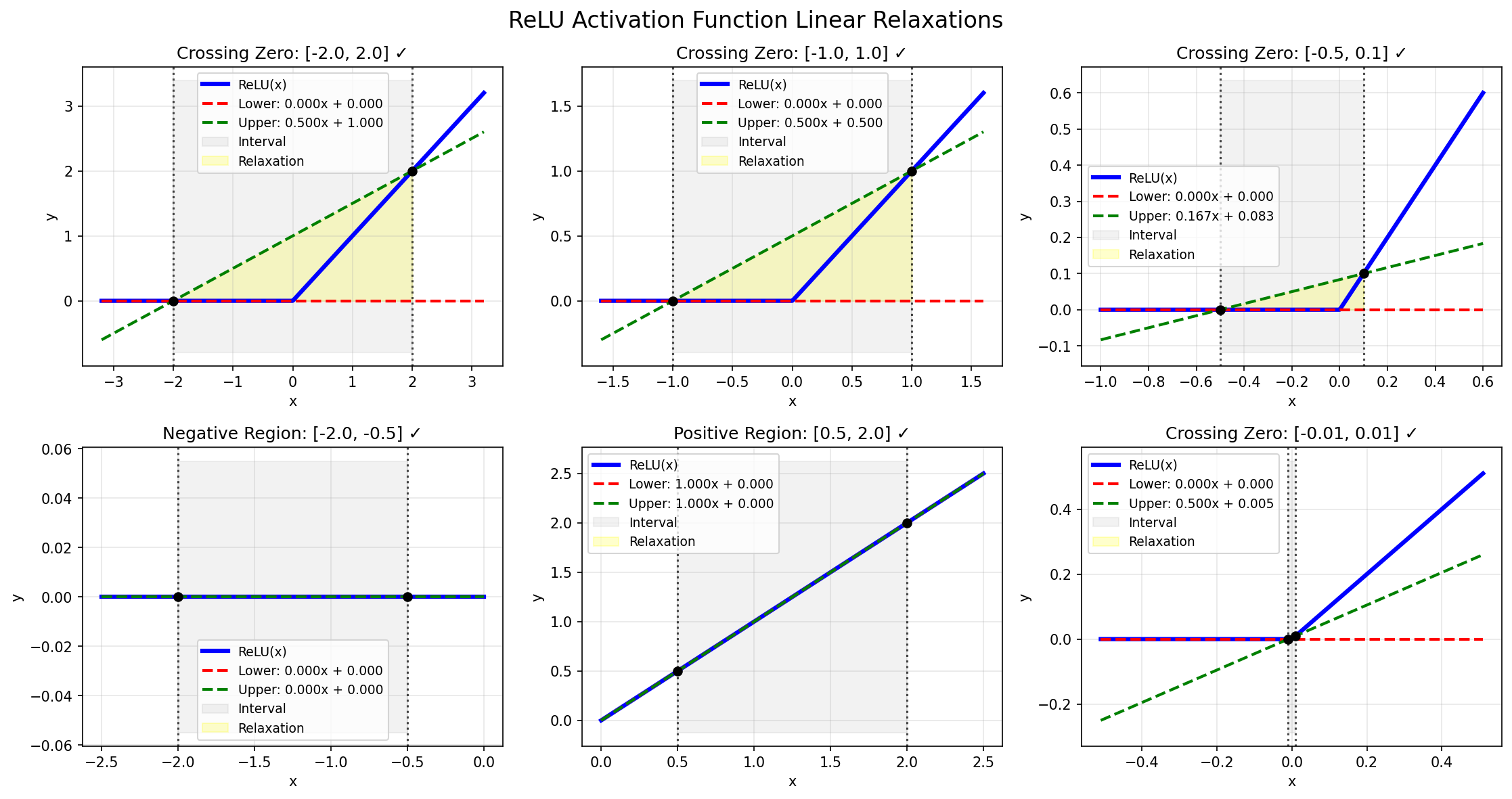}
    \caption{ReLU Relaxation.}
    \label{fig:relu}
\end{figure}
For the Rectified Linear Unit $\sigma(y) = \max(0, y)$, the relaxation depends on whether the interval crosses zero.
\begin{itemize}
  \item \textbf{Active region} ($l\ge 0$):
  \begin{align*}
      &\lowerbound{G}_m = \upperbound{G}_m = 1, \qquad
      &\lowerbound{g}_m = \upperbound{g}_m = 0. \\
      &\lowerbound{S}_m = \upperbound{S}_m = 0, \qquad
      &\lowerbound{s}_m = \upperbound{s}_m = 1.
  \end{align*}

  \item \textbf{Inactive region} ($u\le 0$):
  \begin{align*}
      &\lowerbound{G}_m = \upperbound{G}_m = 0, \qquad
      &\lowerbound{g}_m = \upperbound{g}_m = 0. \\
      &\lowerbound{S}_m = \upperbound{S}_m = 0, \qquad
      &\lowerbound{s}_m = \upperbound{s}_m = 0.
  \end{align*}

  \item \textbf{Unstable region} ($l<0<u$):
  \begin{align*}
      &\lowerbound{G}_m = 0, \quad
      \upperbound{G}_m = \frac{u}{\,u - l\,}, \quad
      &\lowerbound{g}_m = 0, \quad
      &\upperbound{g}_m = -l\,\upperbound{G}_m. \\
      &\lowerbound{S}_m = \upperbound{S}_m = 0, \quad
      &\lowerbound{s}_m = 0, \quad
      &\upperbound{s}_m = 1,
  \end{align*}
\end{itemize}

\subsection*{B. Leaky ReLU Relaxation}
\begin{figure}[h]
    \centering
    \includegraphics[width=0.75\linewidth]{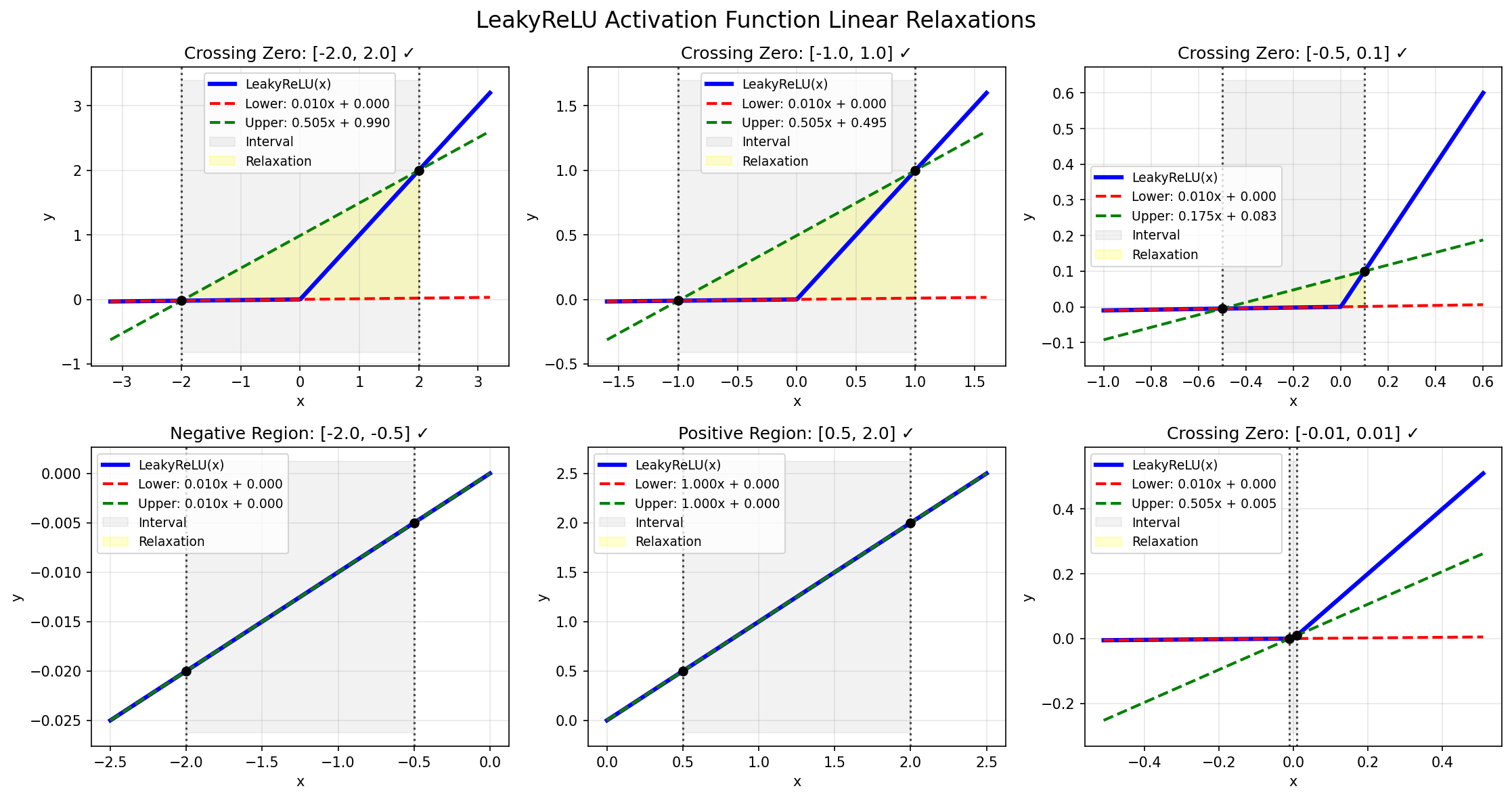}
    \caption{Leaky ReLU Relaxation.}
    \label{fig:leaky_relu}
\end{figure}
For the Leaky ReLU function $\sigma(y) = \max(y, \alpha y)$ with $\alpha \in (0, 1)$:
\begin{itemize}
  \item \textbf{Active region} ($l\ge 0$):
  \begin{align*}
      &\lowerbound{G}_m = \upperbound{G}_m = 1, \qquad
      &\lowerbound{g}_m = \upperbound{g}_m = 0. \\
      &\lowerbound{S}_m = \upperbound{S}_m = 0, \qquad
      &\lowerbound{s}_m = \upperbound{s}_m = 1.
  \end{align*}

  \item \textbf{Negative region} ($u\le 0$):
  \begin{align*}
      &\lowerbound{G}_m = \upperbound{G}_m = \alpha, \qquad
      &\lowerbound{g}_m = \upperbound{g}_m = 0.\\
      &\lowerbound{S}_m = \upperbound{S}_m = 0, \qquad
      &\lowerbound{s}_m = \upperbound{s}_m = \alpha.
  \end{align*}

  \item \textbf{Unstable region} ($l<0<u$):
  \begin{align*}
      &\lowerbound{G}_m = \alpha, \quad
      \upperbound{G}_m = \frac{u - \alpha l}{\,u - l\,}, \quad
      &\lowerbound{g}_m = 0, \quad
      &\upperbound{g}_m = u - \upperbound{G}_m\,u.\\
      &\lowerbound{S}_m = \upperbound{S}_m = 0, \quad
      &\lowerbound{s}_m = \alpha, \quad
      &\upperbound{s}_m = 1,
  \end{align*}
\end{itemize}

\subsection*{C. Sigmoid Relaxation}
\begin{figure}[h]
    \centering
    \includegraphics[width=0.9\linewidth]{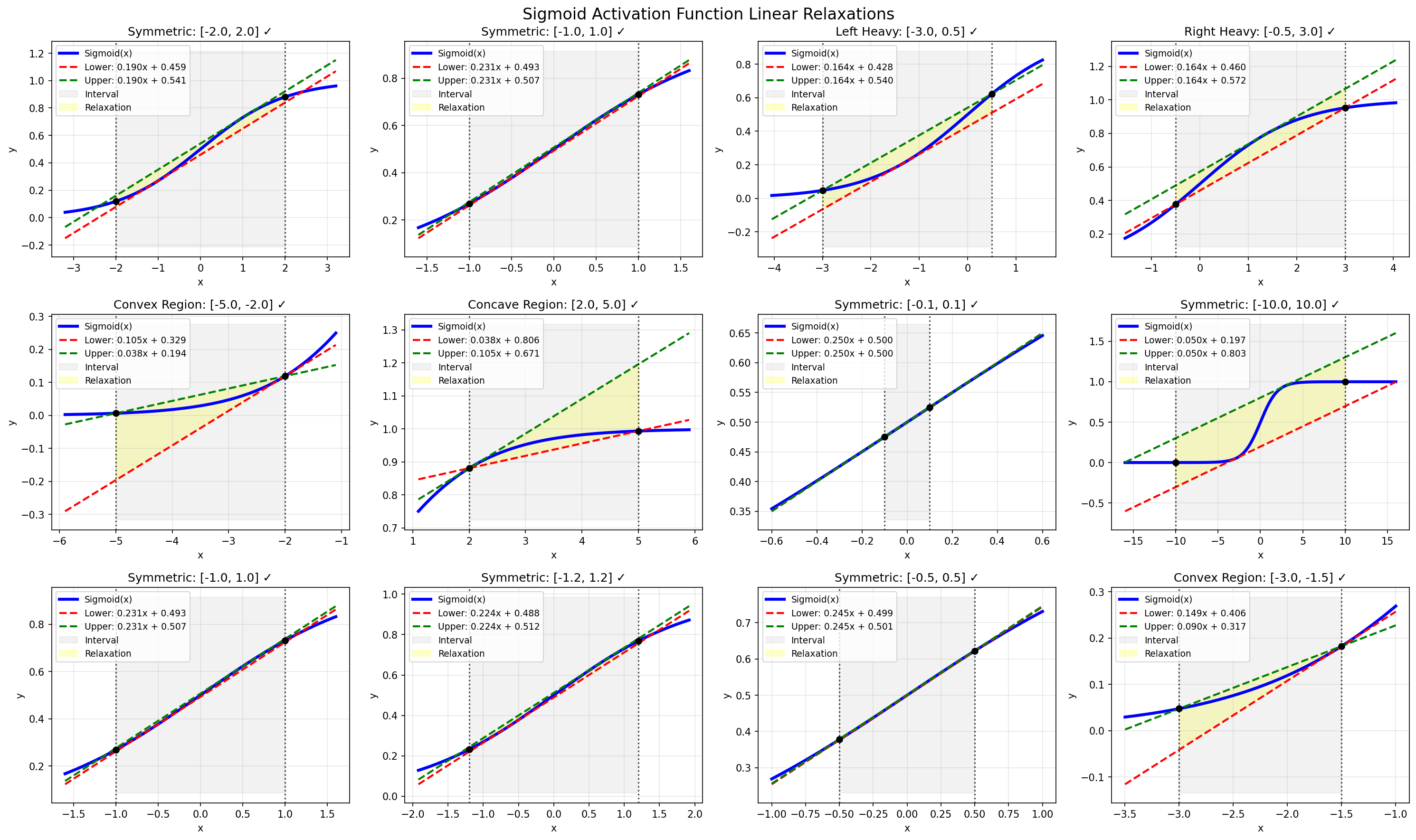}
    \caption{Sigmoid Relaxation.}
    \label{fig:sigmoid}
\end{figure}
\begin{figure}[h]
    \centering
    \includegraphics[width=0.9\linewidth]{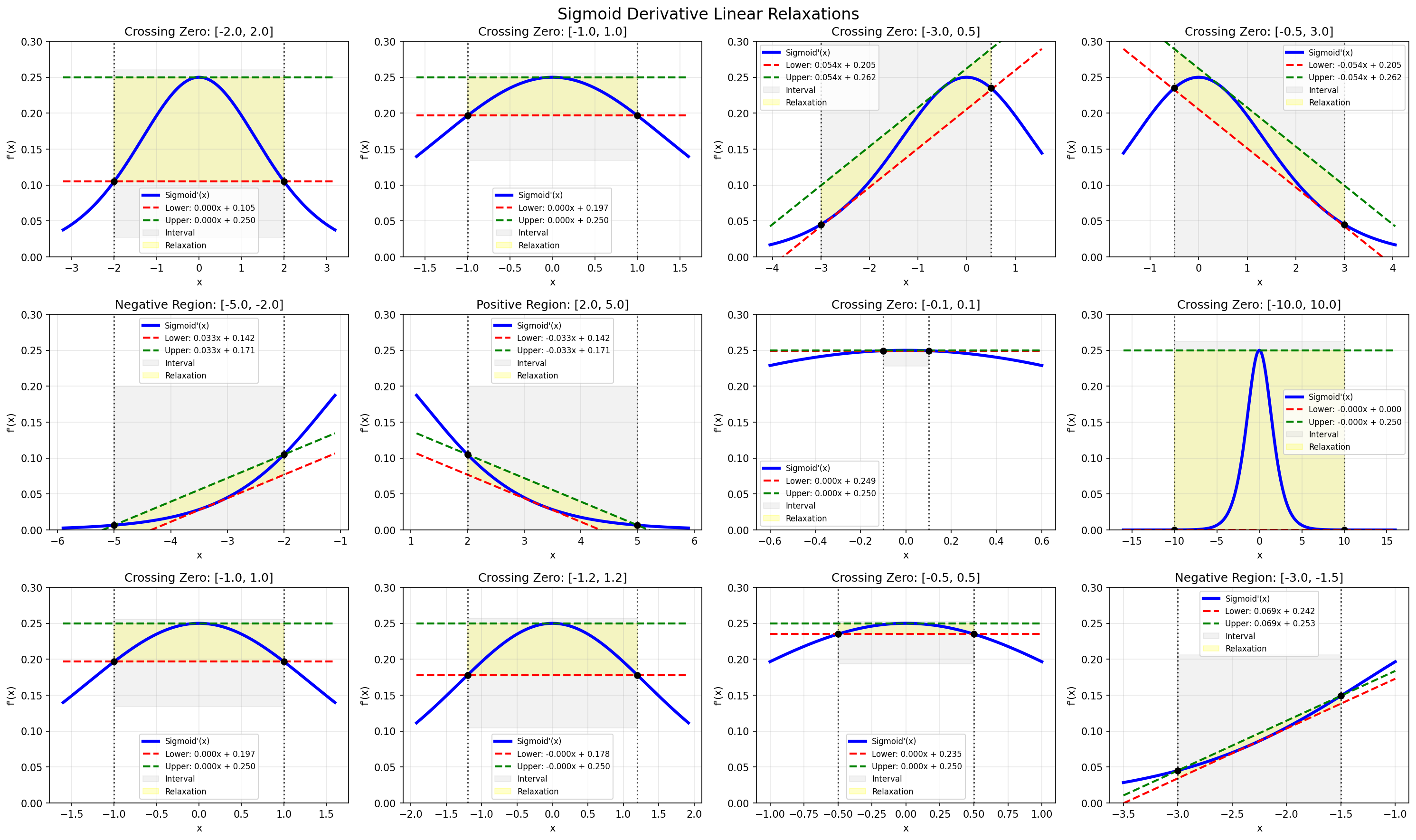}
    \caption{Derivative of Sigmoid Relaxation.}
    \label{fig:derivative_sigmoid}
\end{figure}
The sigmoid activation \(\activationfunc(y) = \frac{1}{1 + e^{-y}} \) is convex for \(y < 0\) and concave for \(y > 0\). The derivative has inflection points at \(\pm \log\!\tfrac{3 + \sqrt{3}}{3 - \sqrt{3}}\) and is concave in between the inflection points, and convex when the domain lies outside the inflection points.
Define
\begin{align*}
    m_{\text{act}} &= \frac{\sigma(u) - \sigma(l)}{u - l}, &
    m_{\text{der}} &= \frac{\activationfunc'(u) - \activationfunc'(l)}{u - l}.
\end{align*}
Let \(t_\star\in(0,1)\) solve \(  2t^3 - 3t^2 + t - m_{\text{der}} = 0\) and define \(x_\star = \log\!\frac{t_\star}{1-t_\star} \in (l,u)\). Then
\begin{itemize}
  \item \textbf{Convex region}:
  \begin{align*}
      \lowerbound{G}_m &= \activationfunc'(u), &
      \upperbound{G}_m &= m_{\text{act}}, &
      \lowerbound{g}_m &= \activationfunc(u) - \activationfunc'(u)\,u, &
      \upperbound{g}_m &= \sigma(l) - m_{\text{act}}\,l, \\
      \lowerbound{S}_m &= m_{\text{der}}, &
      \upperbound{S}_m &= m_{\text{der}}, &
      \lowerbound{s}_m &= \activationfunc'(x_\star) - m_{\text{der}}\,x_\star, &
      \upperbound{s}_m &= \activationfunc'(u) - m_{\text{der}}\,u .
  \end{align*}

  \item \textbf{Concave region}:
  \begin{align*}
      \lowerbound{G}_m &= m_{\text{act}}, &
      \upperbound{G}_m &= \activationfunc'(l), &
      \lowerbound{g}_m &= \sigma(l) - m_{\text{act}}\,l, &
      \upperbound{g}_m &= \activationfunc(l) - \activationfunc'(l)\,l, \\
      \lowerbound{S}_m &= m_{\text{der}}, &
      \upperbound{S}_m &= m_{\text{der}}, &
      \lowerbound{s}_m &= \activationfunc'(l) - m_{\text{der}}\,l, &
      \upperbound{s}_m &= \activationfunc'(x_\star) - m_{\text{der}}\,x_\star .
  \end{align*}

  \item \textbf{Mixed region}:
  Define
  \begin{align*}
      y_\lambda &= \frac{1-\sqrt{\,1-4\,m_{\text{act}}\,}}{2}, &
      y_\mu &= \frac{1+\sqrt{\,1-4\,m_{\text{act}}\,}}{2}, &
      x_\lambda &= \log\!\frac{y_\lambda}{1-y_\lambda}, &
      x_\mu &= -x_\lambda,
  \end{align*}
  and let \(\{t_i\}\subset(0,1)\) be the real roots of
  \(2t^3 - 3t^2 + t - m_{\text{der}}=0\) mapped to
  \(x_i=\log\!\frac{t_i}{1-t_i}\)
  \begin{align*}
      x^- &= \arg\min_{x_i}\ \bigl\{\activationfunc'(x_i) - m_{\text{der}}\,x_i\bigr\}, &
      x^+ &= \arg\max_{x_i}\ \bigl\{\activationfunc'(x_i) - m_{\text{der}}\,x_i\bigr\}.
  \end{align*}
  Then 
  \begin{align}
      \lowerbound{G}_m &= m_{\text{act}}, &
      \upperbound{G}_m &= m_{\text{act}}, &
      \lowerbound{g}_m &= y_\lambda - m_{\text{act}}\,x_\lambda, &
      \upperbound{g}_m &= y_\mu - m_{\text{act}}\,x_\mu , \\
      \lowerbound{S}_m &= m_{\text{der}}, &
      \upperbound{S}_m &= m_{\text{der}}, &
      \lowerbound{s}_m &= \activationfunc'(x^-) - m_{\text{der}}\,x^-, &
      \upperbound{s}_m &= \activationfunc'(x^+) - m_{\text{der}}\,x^+ .
  \end{align}
\end{itemize}

\subsection*{D. Tanh Relaxation}
\begin{figure}[h]
    \centering
    \includegraphics[width=0.9\linewidth]{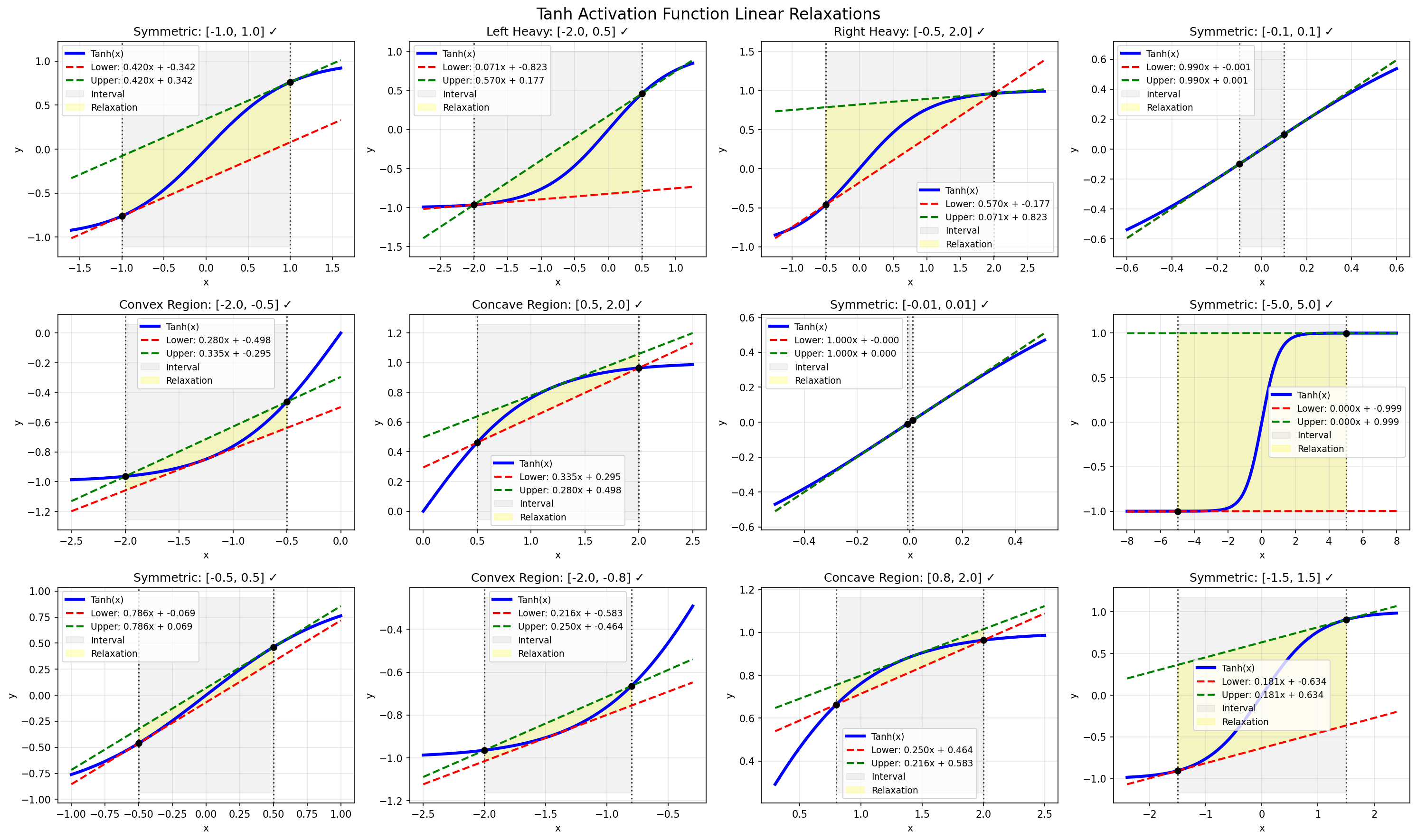}
    \caption{Tanh Relaxation.}
    \label{fig:tanh}
\end{figure}
\begin{figure}[h]
    \centering
    \includegraphics[width=0.9\linewidth]{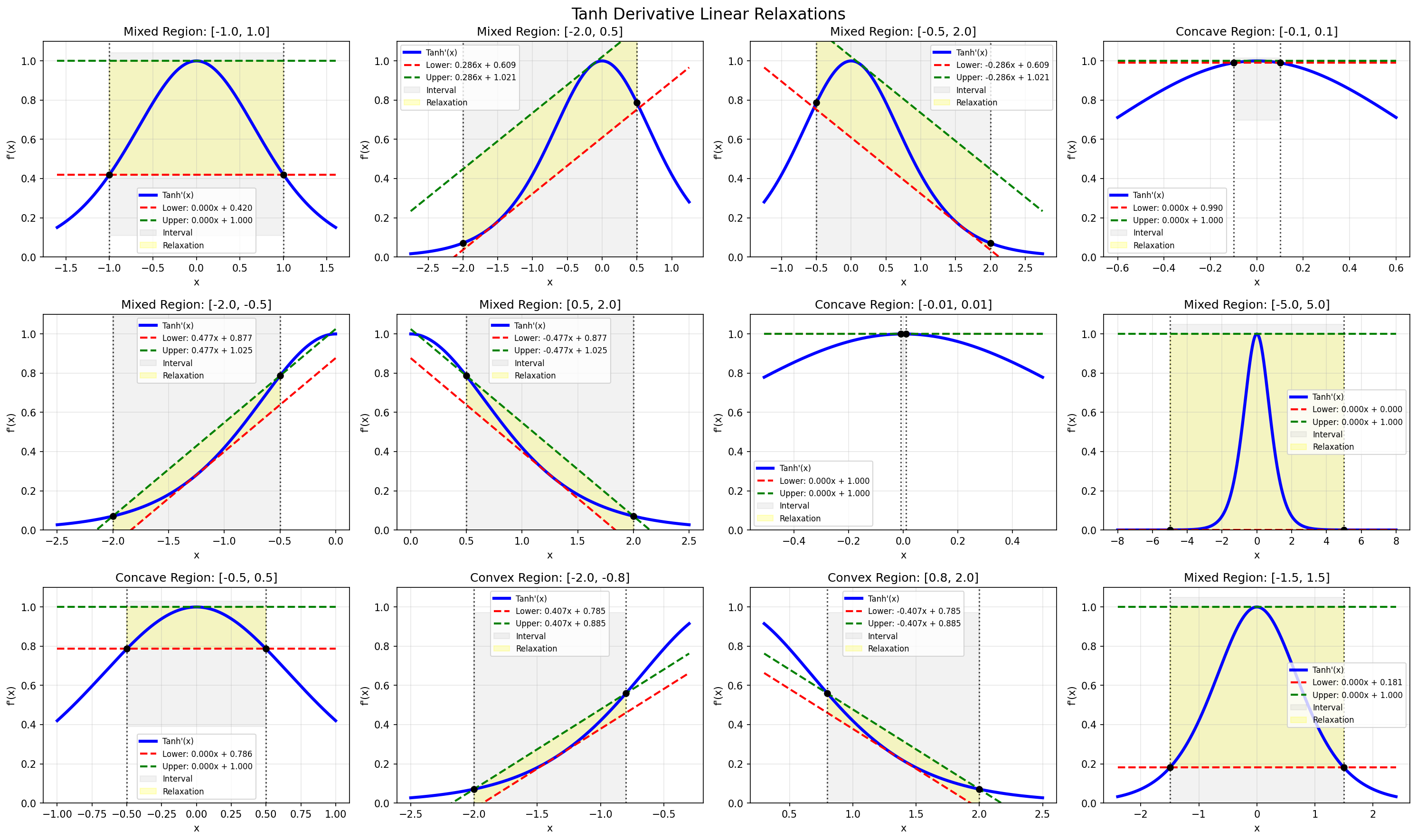}
    \caption{Derivative Tanh Relaxation.}
    \label{fig:derivative_tanh}
\end{figure}

The hyperbolic tangent activation \(\activationfunc(y) = \tanh(y)\) is convex for \(y<0\) and concave for \(y>0\). The derivative has inflection points at \(\pm \arctanh\!\tfrac{1}{\sqrt{3}}\) and is concave in between the inflection points, and convex when the domain lies outside the inflection points.

Define \(m_{\text{act}}=\dfrac{\tanh(u)-\tanh(l)}{u-l}\), \(m_{\text{der}}=\dfrac{\tanh^2(l) - \tanh^2(u)}{u-l}\), and the midpoint \(m=\tfrac{l+u}{2}\).

\begin{itemize}
  \item \textbf{Convex region}:
  \begin{align*}
      \lowerbound{G}_m &= 1-\tanh^2(m),
      &\lowerbound{g}_m &= \tanh(m) - \lowerbound{G}_m\, m, \\
      \upperbound{G}_m &= m_{\text{act}},
      &\upperbound{g}_m &= \tanh(u) - m_{\text{act}}\, u, \\
      \lowerbound{S}_m &= m_{\text{der}},
      &\lowerbound{s}_m &= 1-\tanh^2(x_{\text{tan}}) - m_{\text{der}}\, x_{\text{tan}}, \\
      \upperbound{S}_m &= m_{\text{der}},
      &\upperbound{s}_m &= 1-\tanh^2(l) - m_{\text{der}}\, l.
  \end{align*}

  \item \textbf{Concave region}:
  \begin{align*}
      \lowerbound{G}_m &= m_{\text{act}},
      &\lowerbound{g}_m &= \tanh(l) - m_{\text{act}}\, l, \\
      \upperbound{G}_m &= 1-\tanh^2(m),
      &\upperbound{g}_m &= \tanh(m) - \upperbound{G}_m\, m, \\
      \lowerbound{S}_m &= m_{\text{der}},
      &\lowerbound{s}_m &= 1-\tanh^2(l) - m_{\text{der}}\, l, \\
      \upperbound{S}_m &= m_{\text{der}},
      &\upperbound{s}_m &= 1-\tanh^2(x_{\text{tan}}) - m_{\text{der}}\, x_{\text{tan}}.
  \end{align*}

  \item \textbf{Mixed region}:
  \begin{align*}
      \lowerbound{G}_m &= 1-\tanh^2(l),
      &\lowerbound{g}_m &= \tanh(l) - \lowerbound{G}_m\, l, \\
      \upperbound{G}_m &= 1-\tanh^2(u),
      &\upperbound{g}_m &= \tanh(u) - \upperbound{G}_m\, u, \\
      \lowerbound{S}_m &= m_{\text{der}},
      &\lowerbound{s}_m &= \min_{x\in[l,u]}\bigl\{1-\tanh^2(x) - m_{\text{der}}\, x\bigr\}, \\
      \upperbound{S}_m &= m_{\text{der}},
      &\upperbound{s}_m &= \max_{x\in[l,u]}\bigl\{1-\tanh^2(x) - m_{\text{der}}\, x\bigr\}.
  \end{align*}
\end{itemize}

The tangent points for the derivative envelope are obtained by solving
\begin{align}
    2t^3 - 2t - m_{\text{der}} = 0 \quad \text{for } t=\tanh(x),
    \qquad x_{\text{tan}}=\arctanh(t)\in(l,u),
\end{align}
which specify where the parallel lines of slope \(m_{\text{der}}\) become tight.

\section{Proof of \cref{thm:phi_linear}}\label{appendix:proof}
\begin{proof}
By the bounds $\lowerbound{\NN}(x) \le \NN(x) \le \upperbound{\NN}(x)$, we can certify the sign of $\NN$ from its envelopes. In particular, on $\mathcal{S}^C$, the condition $\upperbound{\NN}(x) < 0$ implies $\NN(x) < 0$. Likewise, on $\mathcal{S}$, the condition $\lowerbound{\NN}(x) > 0$ implies $\NN(x) > 0$.
For invariance, for every $x \in \mathcal{S}$ we have
\begin{align}
    \frac{\partial \NN(x)}{\partial x}\,f(x) &\ge \Gamma_{\mathrm{drift, L}}(\eta)\,x + \beta_{\mathrm{drift, L}}(\eta), \label{eq:term1}\\
    \sup_{u\in\mathcal{U}} \frac{\partial \NN(x)}{\partial x}\,g(x)u 
    &\ge \Gamma_{\mathrm{ctrl, L}}(\eta)\,x + \beta_{\mathrm{ctrl, L}}(\eta), \label{eq:term2}\\
    \alpha\,\NN(x) &\ge \alpha\,\lowerbound{\Pi}\,x + \lowerbound{\pi}. \label{eq:term3}
\end{align}
Adding \cref{eq:term1,eq:term2,eq:term3} yields the left-hand side of $\psi_{\mathrm{invar}}$. Since the right-hand side is nonnegative on $\mathcal{S}$ for all $x \in \mathcal{S}$, it follows that
\begin{align}
    \frac{\partial \NN(x)}{\partial x}\,f(x)
    +\sup_{u\in\mathcal{U}} \frac{\partial \NN(x)}{\partial x}\,g(x)u
    +\alpha\,\NN(x)
    &\ge 0.
\end{align}
\end{proof}

\section{Benchmarks} \label{appendix:benchmarks}
\subsection{2D-Control}
\begin{figure}[h]
    \centering
    \includegraphics[width=1\linewidth]{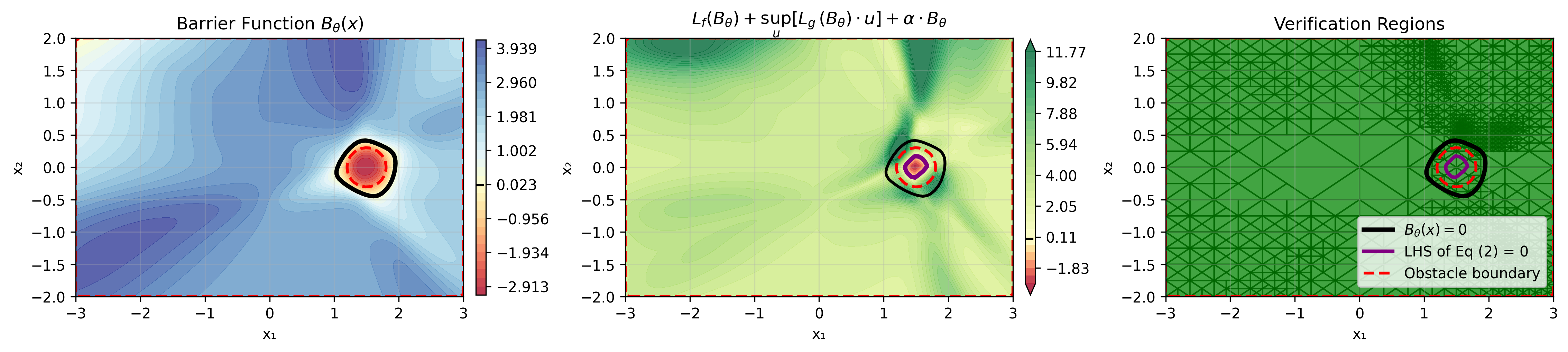}
    \caption{2D-Control: Illustration of the barrier function \(\NN(x)\), the CBF invariance condition, and the resulting verification regions after adaptive mesh refinement.}
    \label{fig:mesh_refinement_simple2D}
\end{figure}
States $x=(x_1,x_2)\in\reals^2$ with affine control $u=(u_1,u_2)$ and dynamics:
\begin{equation}\label{eq:2d-control}
    \dot{x} = \begin{bmatrix} -x_1 x_2 \\-x_2^2 \end{bmatrix} + \begin{bmatrix} 1 & 0 \\0 & 1 \end{bmatrix}\,u,
\end{equation} 
Control bounds are \(u \in [-\frac{1}{2},\frac{1}{2}]^2\). The state space is \(\statespace = [-3, 3]\times [-2, 2]\), with the safe set defined as \(\mathcal{S} = \statespace \backslash  \{(x,y)\in\reals^2 \mid \sqrt{(x-1.5)^2 + y^2} \le 0.3\}\).

\subsection{Cart-Pole}
The dynamics for the cart pole are taken from \citep{Kimura1999} with states \(x=(y,\dot{y},\theta,\dot{\theta})\in\reals^4\) and single-input affine control \(u=F\in\reals\). The parameters are the cart mass \(m_c\), pole mass \(m_p\) and length \(L\), gravitational acceleration \(g\) and pole-friction coefficient \(\mu_p\). The cart-track friction \(\mu_c\) is neglected.

Define
\begin{align}
   \ddot{\theta}(x)
  &= \frac{ g\sin\theta \;+\; \cos\theta\!\left(-\,m_p L\,\dot{\theta}^2\sin\theta\right)\!/(m_c+m_p)\;-\;\mu_p\,\dot{\theta}/(m_p L) }{ L\!\left(\tfrac{4}{3} \;-\; \dfrac{m_p \cos^2\theta}{m_c+m_p}\right)} ,\\[4pt]
  \ddot{y}(x)
  &= \frac{ m_p L\!\left(\dot{\theta}^2\sin\theta \;-\; \ddot{\theta}(x)\cos\theta\right)}{m_c+m_p}.
\end{align}
Then
\begin{equation}
  \dot{x}\;=\;
  \begin{bmatrix}
    \dot{y}\\[2pt]
    \ddot{y}(x)\\[2pt]
    \dot{\theta}\\[2pt]
    \ddot{\theta}(x)
  \end{bmatrix} + 
  \begin{bmatrix}
    0\\
    \frac{1 - m_p L \cos\theta \,  g_{\ddot{\theta}}(x)}{m_c+m_p}\\
    0\\
    g_{\ddot{\theta}}(x)
  \end{bmatrix}u,
\end{equation}
where
\begin{equation}
  g_{\ddot{\theta}}(x)
  \;=\;
  -\,\frac{\cos\theta}{(m_c+m_p)\,L\!\left(\tfrac{4}{3}-\dfrac{m_p \cos^2\theta}{m_c+m_p}\right)}.
\end{equation}

The control bounds are \(u \in [-u_{\max},\,u_{\max}]\), with \(u_{\max} = 10~\text{N}\). The state space is
\begin{equation}
  \statespace \;=\; [-2.4,\,2.4] \times [-3,\,3] \times \Big[-\tfrac{\pi}{6},\,\tfrac{\pi}{6}\Big] \times [-2,\,2] \;\subset\; \reals^4 .
\end{equation}

The safe set constrains only the cart position \(y\) (within the state space bounds for the other coordinates): \(\mathcal{S} \;=\; \big\{(y,\dot{y},\theta,\dot{\theta}) \in \statespace \;\big|\; y \in [-2,\,2]\big\}\). Parameters used in the benchmark are \(m_c = 1.0~\mathrm{kg}\), \(
  m_p = 0.1~\mathrm{kg}\), \(
  L = 0.5~\mathrm{m}\), \(
  g = 9.81~\mathrm{m/s^2}\), \(
  \mu_p = 0.01
\).

\subsection{Darboux (Barrier 1 in \cite{abate_formal_2021})}
\begin{figure}[h]
    \centering
    \includegraphics[width=1\linewidth]{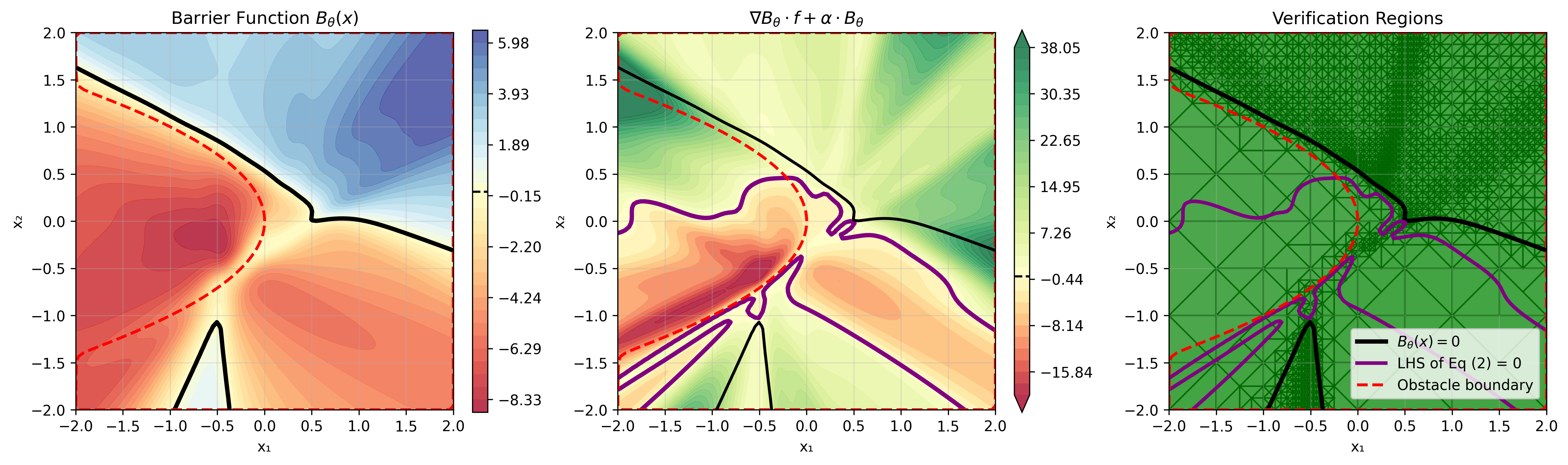}
    \caption{Darboux: Illustration of the barrier function \(\NN(x)\), the CBF invariance condition, and the resulting verification regions after adaptive mesh refinement.}
    \label{fig:mesh_refinement_darboux}
\end{figure}
Originally presented in \cite{Zeng_2016}, the system has been reported to not admit an LMI-based barrier function with a degree less than 6 \citep{abate_fossil_2021}, making it a common benchmark to demonstrate the expressivity of neural CBFs. Using a large network and \texttt{tanh} activation functions, we are able to obtain a larger control invariant when compared to e.g. \cite{zhang_exact_nodate}.

States $x=(x_1,x_2)\in\reals^2$ and autonomous dynamics:
\begin{equation}\label{eq:darboux}
\dot{x} = \begin{bmatrix}
x_2 + 2 x_1 x_2 \\[6pt]
-\,x_1 - x_2^2 + 2 x_1^2
\end{bmatrix}.
\end{equation}
The state space is \(\statespace = [-2, 2]\times [-2, 2]\), with the safe set defined as \(\mathcal{S} = \statespace \backslash  \{(x_1,x_2)\in\reals^2 \mid x_1 + x_2^2 \le 0 \}\).

\subsection{Barrier 2}
\begin{figure}[h]
    \centering
    \includegraphics[width=1\linewidth]{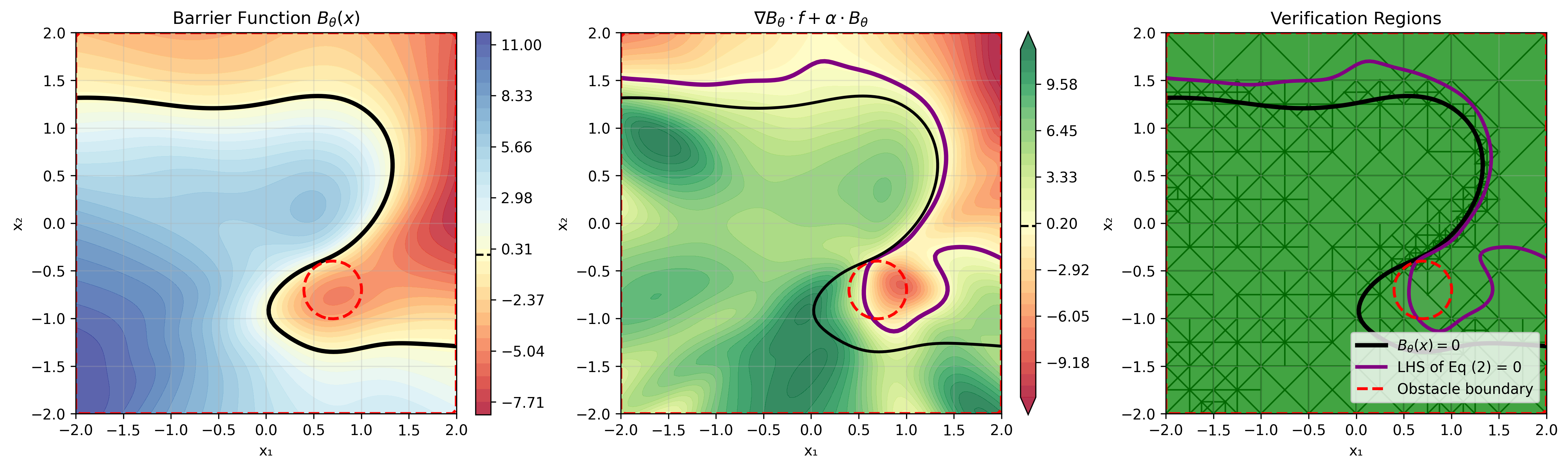}
    \caption{Barrier 2: Illustration of the barrier function \(\NN(x)\), the CBF invariance condition, and the resulting verification regions after adaptive mesh refinement.}
    \label{fig:mesh_refinement_barrier2}
\end{figure}
Presented in \cite{Jiang2015} and chosen for its high degree of nonlinearity and non-polynomial (exponential and trigonometric) terms.
States $x=(x_1,x_2)\in\reals^2$ and autonomous dynamics:
\begin{equation}\label{eq:barrier2}
    \dot{x_1} = \begin{bmatrix}
    e^{-x_1} + x_2 - 1 \\[6pt]
    - \sin^2(x_1)
    \end{bmatrix},
\end{equation}
The state space is \(\statespace = [-2, 2]\times [-2, 2]\), with the safe set defined as \(\mathcal{S} = \statespace \backslash  \{(x,y)\in\reals^2 \mid \sqrt{(x-0.7)^2 + (y+0.7)^2} \le 0.3\}\).

\subsection{Barrier 3}
\begin{figure}[h]
    \centering
    \includegraphics[width=1\linewidth]{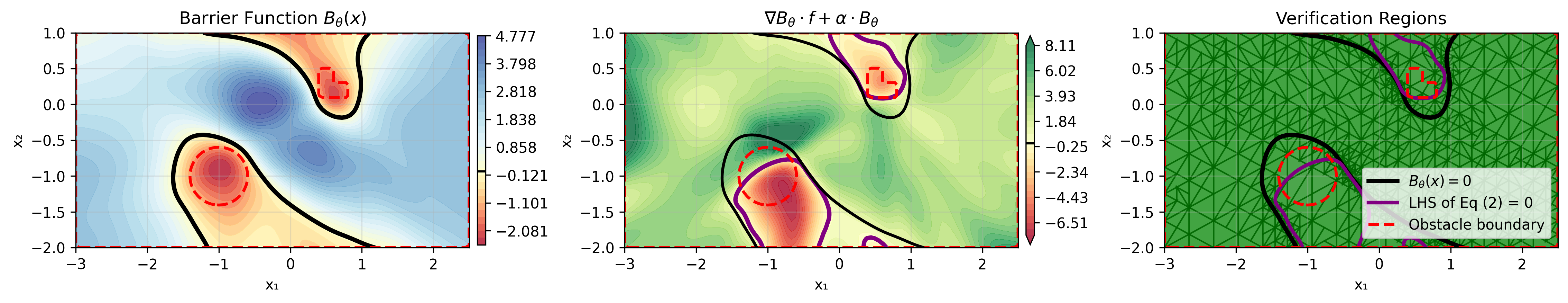}
    \caption{Barrier 3: Illustration of the barrier function \(\NN(x)\), the CBF invariance condition, and the resulting verification regions after adaptive mesh refinement.}
    \label{fig:mesh_refinement_barrier3}
\end{figure}
A modification of the dynamical system presented in \cite{PRAJNA2006117} with non-convex domains, as in \cite{abate_fossil_2021}.
States $x=(x_1,x_2)\in\reals^2$ and autonomous dynamics:
\begin{equation}\label{eq:barrier3}
\dot{x} = \begin{bmatrix}
x_2 \\
- x_1 - x_2 + \tfrac{1}{3}x_1^3
\end{bmatrix}.
\end{equation}
The state space is \(\statespace = [-3, 2.5]\times [-2, 1]\), with the safe set defined as \(\mathcal{S} = \statespace \backslash \statespace_{\mathrm{unsafe}}\). The unsafe safe set \(\statespace_{\mathrm{unsafe}}\) is given by two obstacles, a disk and an L-shaped region. The circle is \(\{(x,y)\mid (x+1)^2+(y+1)^2 \le 0.4^2\} \) while the L-shaped region is given by the union of two axis-aligned rectangles, \(\left[[0.4, 0.6]\times [0.1, 0.5]\right] \cup \left[[0.4, 0.8]\times [0.1, 0.3]\right]\).

\subsection{Barrier 4}
An Unmanned Aerial Vehicle (UAV) avoiding collision with an obstacle as presented in \citep{Barry2012, zhang_exact_nodate}.
States $\mathbf{x}=(x,y,\phi)\in\reals^3$. The dynamics are
\begin{equation}\label{eq:barrier4}
    \dot{\mathbf{x}} =
    \begin{bmatrix}
    v \sin\phi \\
    v \cos\phi \\
    - \sin\phi \;+\; 3\,\dfrac{x\sin\phi + y\cos\phi}{\,0.5 + x^2 + y^2\,}
    \end{bmatrix},
\end{equation}
with $v=1$. The last component represents a heading-rate term that depends on the current pose $(x,y,\phi)$. The state space is \(\statespace = [-2, 2]\times [-2, 2]\times[-\tfrac{\pi}{2},\tfrac{\pi}{2}]\), with the safe set defined as \(\mathcal{S} = \statespace \backslash  \{(x,y)\in\reals^2 \mid \sqrt{x^2 + y^2} \le 0.2\}\).

\section{Taylor Expansions of Elementary Functions over a Simplex}\label{appendix:taylor}
As certified Taylor expansions have been discussed in the literature, particularly in \citep{mathiesen_certified_2025}, we only briefly mention the extensions required to bound the expansion over a simplex. Let $f : \mathcal{X} \!\to\! \reals^n$ be continuously differentiable, and let the domain
of interest be a simplex
\[
\simplex = \mathrm{conv}\{v_0, v_1, \dots, v_n\}
  = \Bigl\{\, x = \sum_{i=0}^{n}\lambda_i v_i
    \;\Big|\;
    \lambda_i \ge 0,\;
    \sum_i \lambda_i = 1
  \Bigr\}.
\]
The barycentric coordinates $\lambda_i$ parameterize any point $x \in \simplex$ as an affine
combination of its vertices. We chose the expansion point $c \in \simplex$ to be the barycenter
$c = \tfrac{1}{n+1}\sum_i v_i$. The Taylor expansion of $f$ around~$c$ is \(f(x) = f(c) + \nabla_x f(c)\,(x - c) + R(x)\), where the Lagrange remainder satisfies \(R(x) = \tfrac{1}{2}(x-c)^{\!\top}\nabla_x^2 f(\xi_x)(x-c)\), for \(\xi_x \in \simplex\).
To certify the approximation, we construct bounds $[R_{\min},R_{\max}]$ enclosing all
possible values of $R(x)$ over~$\simplex$. To this end, we consider the quadratic form $q(x)=(J(x-c))^2$ arising in the remainder above, where
$J=\nabla_x f(c)$. Over a simplex, the range of any polynomial can be bounded exactly by
its \emph{Bernstein coefficients}.
Writing a degree-$d$ polynomial $p(x)$ in barycentric form,
\[
p(x)=\sum_{|\alpha|=d}b_\alpha\,B_\alpha^d(\lambda(x)),
\qquad
B_\alpha^d(\lambda)
   =\frac{d!}{\alpha_0!\dots\alpha_n!}
      \lambda_0^{\alpha_0}\!\dots\lambda_n^{\alpha_n},
\]
we obtain \( \min_\alpha b_\alpha \le p(x) \le \max_\alpha b_\alpha\), for \(x\in\simplex.\)
For the quadratic term, the degree-2 Bernstein coefficients can be computed directly from
the vertex evaluations of the linear form $L(x)=J(x-c)$:
\[
b_\alpha(q)=
\sum_{|\beta|=1,|\gamma|=1}b_\beta(L)\,b_\gamma(L),
\]
where each $b_\beta(L)$ equals the value of $L$ at the corresponding vertex~$v_i$.
This yields the certified remainder bounds
\[
R_{\min}=\tfrac{1}{2}\min_\alpha b_\alpha(q)\,M_{\min}(f''),
\qquad
R_{\max}=\tfrac{1}{2}\max_\alpha b_\alpha(q)\,M_{\max}(f''),
\]
with $M_{\min}(f'')$ and $M_{\max}(f'')$ denoting lower and upper bounds on the Hessian
entries of~$f$ over~$\simplex$.

\end{document}